\documentclass[twoside,11pt]{article}

\usepackage{multirow}
\usepackage{booktabs}

\usepackage{bm}
\usepackage{tcolorbox}
\usepackage{algorithm}
\usepackage{algorithmic}
\usepackage{xcolor}  % 颜色支持
\usepackage{amsmath}
\usepackage{amsfonts}

\usepackage{xcolor}

\usepackage{graphicx}
\usepackage{subcaption}

\usepackage{array}
\usepackage{booktabs}
\usepackage{enumitem}
\usepackage[abbrvbib, preprint]{ArXiv}
\usepackage[svgnames]{xcolor}
\hypersetup{
hidelinks,
colorlinks=true,
linkcolor=Indigo,
urlcolor=DeepSkyBlue4,
citecolor=Indigo
}
\usepackage{authblk}
\usepackage{amsmath} 

\def\vx{{\boldsymbol{x}}}
\def\vz{{\boldsymbol{z}}}
\DeclareMathOperator*{\argmax}{arg\,max}

\DeclareFixedFont{\myfont}{OT1}{ptm}{m}{n}{8pt}
\DeclareFixedFont{\myfontb}{OT1}{ptm}{bx}{n}{8pt}

\usepackage{marvosym}
\makeatletter
\def\@fnsymbol#1{\ifcase#1\or \text{\Letter}\or *\or \dagger\or \ddagger\else\@arabic{#1}\fi}
\makeatother
\usepackage{lastpage}
\jmlrheading{23}{\ Preprint 2026}{1-\pageref{LastPage}}{1/19}{2/5}{21-0000}{}

\firstpageno{1}

\begin{document}

\title{On the Learnability of Offline Model-Based Optimization: \\A Ranking Perspective}

% \author{\name Shen-Huan Lyu~\thanks{Corresponding author} \email lvsh@hhu.edu.cn \\
%        \addr Key Laboratory of Water Big Data Technology of Ministry of Water Resources,\\ College of Computer Science and Software Engineering, Hohai University, Nanjing, China\\
%        \addr Department of Computer Science, City University of Hong Kong, Hong Kong, China\\
%        \addr State Key Laboratory for Novel Software Technology,
%        Nanjing University, Nanjing, China
%        % \AND
%        % \name Author Two \email two@cs.berkeley.edu \\
%        % \addr Division of Computer Science\\
%        % University of California\\
%        % Berkeley, CA 94720-1776, USA
%        }

\author[1,2,3]{Shen-Huan Lyu}
\author[3]{Rong-Xi Tan}
\author[3]{Ke Xue}
\author[4]{Yi-Xiao He}
\author[1]{Yu Huang}
\author[2]{Qingfu Zhang}
\author[3,\thanks{Corresponding author: qianc@nju.edu.cn}]{Chao Qian}

\affil[1]{Key Laboratory of Water Big Data Technology of Ministry of Water Resources, College of Computer Science and Software Engineering, Hohai University, Nanjing, China}
\affil[2]{Department of Computer Science, City University of Hong Kong, Hong Kong, China}
\affil[3]{National Key Laboratory for Novel Software Technology, Nanjing University, Nanjing, China}
\affil[4]{School of Artificial Intelligence and Information Technology, Nanjing University of Chinese Medicine, Nanjing, China}

\editor{My editor}

\maketitle

\begin{abstract}
Offline model-based optimization (MBO) seeks to discover high-performing designs using only a fixed dataset of past evaluations. Most existing methods rely on learning a surrogate model via regression and implicitly assume that good predictive accuracy leads to good optimization performance. In this work, we challenge this assumption and study offline MBO from a learnability perspective. We argue that offline optimization is fundamentally a problem of ranking high-quality designs rather than accurate value prediction. Specifically, we introduce an optimization-oriented risk based on ranking between near-optimal and suboptimal designs, and develop a unified theoretical framework that connects surrogate learning to final optimization. We prove the theoretical advantages of ranking over regression, and identify distributional mismatch between the training data and near-optimal designs as the dominant error. Inspired by this, we design a distribution-aware ranking method to reduce this mismatch. Empirical results across various tasks show that our approach outperforms twenty existing methods, validating our theoretical findings. Additionally, both theoretical and empirical results reveal intrinsic limitations in offline MBO, showing a regime in which no offline method can avoid over-optimistic extrapolation.
\end{abstract}


\section{Introduction}

% **问题背景与应用动机** 基于模型的离线优化（MBO）是指仅利用预先收集的静态设计评估数据集，寻找能够最大化黑盒目标函数输入方案的问题。这一范式在涉及高成本或受限实验的应用中日益重要，尤其在蛋白质工程、材料发现和机械设计领域。在这些领域中，优化蛋白质序列、化学分子或机械结构等设计通常需要耗费大量实验成本，包括物理合成和测试分子或材料。离线MBO通过利用现有数据集推断有潜力的候选方案，而非通过新实验查询目标值，提供了一种数据驱动的解决方案。通过将搜索范围限制在现有数据包含的信息内，离线MBO使得在无法进行迭代在线实验的情况下仍能实现设计优化。该方法已成功应用于解决科学和工业领域的复杂黑盒优化问题，通过复用历史数据提出改进设计方案来加速发现进程。
Offline model-based optimization (MBO) refers to the problem of finding an input design that maximizes a black-box objective function using only a static, pre-collected dataset of evaluated designs~\citep{trabucco2022design,kim2025offline,qiansoo}. 
This paradigm has grown in importance for applications involving costly or limited experiments, particularly in protein engineering~\citep{reddy2024designing,chen2023bidirectional,takizawa2025safe}, materials discovery~\citep{shin2025offline}, and mechanical design~\citep{wang2025model}.
In these fields, optimizing a design such as a protein sequence, a chemical molecule, or a mechanical structure generally demands costly experimental efforts, including physically synthesizing and testing a molecule or material.
Offline MBO provides a data-driven solution by leveraging existing datasets to infer promising candidates, rather than querying the objective through new experiments.
By restricting search to the information contained in available data, offline MBO enables design optimization in situations where iterative online experimentation is infeasible. 
It has been applied successfully to solve complex black-box optimization problems in science and industry, accelerating discovery by reusing historical data to propose improved designs~\citep{hardware,shin2025offline}.

%% 此处安排一张插图来展示主流方法的问题
%\begin{figure}[t]
%	\centering
%	\includegraphics[width=\linewidth]{figs/offline_mbo.pdf}
%	\caption{Caption}
%	\label{fig:offline_mbo}
%\end{figure}

% **主流方法与现有可学习性假设** 基于这一范式，现有的大多数离线MBO方法都围绕通过逐点回归进行代理学习展开，即训练一个模型通过最小化离线数据集上的均方误差（MSE）来逼近未知目标函数~\citep{fu2021offline,qi2022data,chen2023bidirectional,hoang2024learning,dao2024boosting}。这一设计选择隐含了一个强可学习性假设：如果代理模型在现有数据上实现了足够低的MSE，那么优化该代理模型将产生高质量的设计。在这种观点下，离线MBO的失败通常归因于优化过程中遇到的分布外（OOD）区域中代理模型的泛化能力不足~\citep{gulrajani2021in}。因此，大量研究致力于减少OOD回归误差，例如通过显式正则化惩罚远离数据流形的自信预测~\citep{fu2021offline,yu2021roma}，基于集成的方法通过多个代理模型取平均来降低方差~\citep{yuan2023importance,chen2023parallel}，或采用不确定性感知优化策略避免利用高不确定性区域~\citep{fu2021offline,trabucco2021conservative}。尽管这些方法带来了实证上的改进，但它们基本保留了相同的核心假设：控制MSE足以实现可靠的优化。
Building upon this paradigm, most existing offline MBO methods center around surrogate learning via pointwise regression, where a model is trained to approximate the unknown objective function by minimizing mean squared error (MSE) on the offline dataset~\citep{fu2021offline,qi2022data,chen2023bidirectional,hoang2024learning,dao2024boosting}. Implicit in this design choice is a strong learnability assumption: If the surrogate achieves sufficiently low MSE on the available data, then optimizing this surrogate will yield high-quality designs. Under this view, failures of offline MBO are often attributed to poor generalization of the surrogate in out-of-distribution (OOD) regions encountered during optimization~\citep{gulrajani2021in}. Consequently, a large body of work has focused on mitigating OOD regression error. For example, through explicit regularization that penalizes confident predictions far from the data manifold~\citep{fu2021offline,yu2021roma}, ensemble-based methods that average over multiple surrogates to reduce variance~\citep{yuan2023importance,chen2023parallel}, or uncertainty-aware optimization strategies that discourage exploitation of high-uncertainty regions~\citep{fu2021offline,trabucco2021conservative}. While these approaches have led to empirical improvements, they largely preserve the same underlying assumption that controlling MSE is sufficient for reliable optimization.

% **问题的症结：** 然而，这种主流观点掩盖了一个更本质的问题：代理训练目标与离线MBO实际目标之间存在错位。离线MBO的目的并非精确预测设计空间中的函数值，而是找出具有最高真实目标值的设计方案。从这个意义上说，均方误差（MSE）通过惩罚绝对预测偏差来追求分数精确度，而离线MBO本质上依赖于分数排序——尤其是顶尖或高百分位设计的相对排序。对优化器而言，将接近最优的候选方案正确排序于明显次优方案之上，远比最小化逐点预测误差更重要。近期实证研究显示，较低的MSE未必带来更好的优化性能，尤其在搜索进入分布外区域时。在这些情况下，以保持相对排序为目标训练的代理模型往往优于MSE训练的模型，这表明离线MBO的有效性更多取决于维持高质量设计间正确排序的能力，而非全局回归精度。
However, this prevailing perspective obscures a more fundamental issue: A mismatch exists between the surrogate training objective and the actual goal of offline MBO. The purpose of offline MBO is not to accurately predict function values over the design space but to identify designs with the highest true objective values. In this sense, MSE focuses on score precision by penalizing deviations in absolute predictions, whereas offline MBO fundamentally depends on score ranking, especially among top-performing or high-percentage designs. \citet{tan2024offline} explicitly introduce ranking objectives and empirically demonstrate their advantages. From the optimizer's perspective, correctly ranking near-optimal candidates above clearly suboptimal ones is more important than minimizing pointwise prediction error. ROOT~\citep{dao2025root} reframe offline MBO as a distributional translation from low-value designs to high-value designs and establish a new SOTA performance. These works suggest that the effectiveness of offline MBO is governed more by the ability to maintain correct ranking among high-quality designs than by global regression accuracy.

% **研究空白与理论动机** 尽管有这些观察，离线MBO的可学习性仍然知之甚少。虽然最近的一些研究试图超越标准的均方误差回归范式~\citep{krishnamoorthy2023diffusion,tan2024offline,dao2025root}，但它们大多停留在算法层面，缺乏一个原则性的理论框架来解释分布偏移如何影响优化，或者替代训练风险如何与最终设计质量相关联。特别是，如何以与优化相关的方式描述分布外现象仍然不清楚。
Despite these observations, the learnability of offline MBO remains poorly understood. Although several recent works attempt to move beyond the standard MSE regression paradigm~\citep{pan2024model,tan2024offline,dao2025root,zhou2026learning}, such efforts remain largely algorithmic and lack a principled theoretical framework that explains how distribution shift impacts optimization or how surrogate training risk relates to final design quality.

% 本研究提出了一个关于离线基于模型优化的统一理论视角，既阐明了其算法优势，也揭示了根本性局限。我们首先证明：在面向优化的风险定义下，排序目标函数比回归损失具有更严格的泛化保证，这为排序方法在离线MBO中的实证成功提供了理论解释。随后，我们发现训练数据与近最优设计方案之间的分布失配是优化误差的主要来源，并通过数据分布重塑系统性地降低这种失配，从而提升基于排序的优化效果。最后，我们通过将不可约简的泛化误差与近优解和训练数据支撑集之间的几何分离相关联，揭示了离线MBO的内在局限，发现存在某些情况下任何离线方法都无法避免过度乐观的外推。在多样化设计任务上的大量实验验证了理论预测，并阐明了这些结论的实际意义。
In this work, we propose a unified theoretical framework for offline MBO that directly connects surrogate learning to final optimization. Within this framework, we establish three key theoretical contributions.
\begin{itemize}[itemsep=0pt, topsep=0pt, parsep=0pt] % 将相关间距设为0
    \item We show that ranking-based objectives admit strictly tighter generalization guarantees than regression losses, explaining why accurate prediction is not sufficient for successful offline optimization.
    \item We identify distributional mismatch between the training data and near-optimal designs as the dominant source of optimization error, and show that this error can be reduced by appropriately shaping the effective training data distribution.
    \item We characterize an intrinsic limitation of offline MBO by relating the irreducible error to the geometric separation between near-optimal solutions and the support of the training data, revealing regimes in which offline optimization is fundamentally unreliable.
\end{itemize}

Inspired by these theoretical insights, we develop a \textbf{D}istribution-\textbf{A}ware \textbf{R}anking (\textbf{DAR}) method to reduce the mismatch between the training data and near-optimal designs, which improves offline MBO performance. Extensive experiments across diverse tasks validate both the advantages and the intrinsic limitations predicted by our theory.

\section{Preliminary}\label{sec:preliminary}

% 后期可以考虑简化此章节的介绍，重点讲解公式

% 首先，我们对offline MBO进行形式化的描述。Offline MBO通常可以抽象成下面两个阶段，第一阶段利用离线数据集学习代理模型，第二阶段在代理模型上搜索最优参数。整个过程都是在给定的离线数据集上完成，不与真实目标函数交互。
% We provide a formal description of offline MBO. Offline MBO can generally be abstracted into the following two-stage process: Firstly, learning the surrogate from offline data to capture the input–output relationship. Secondly, performing optimization over this surrogate model to infer promising candidates expected to yield high values of the true objective while accounting for distributional and epistemic uncertainty due to the limited and potentially biased offline dataset.
We present a formal characterization of offline MBO. The framework can generally be formulated as a two-stage process: (1) learning a surrogate model from offline data to capture the input–output mapping; (2) performing an optimization over the learned surrogate to infer candidate solutions that are expected to achieve high values under the true objective, while accounting for both distributional shift and epistemic uncertainty arising from the limited and potentially biased offline dataset.

\paragraph{Problem setting:}
In the offline MBO setting, we consider a decision space $\mathcal{X} \subseteq \mathbb{R}^d$ and an unknown objective function $f: \mathcal{X} \rightarrow \mathbb{R}$. Direct evaluations of $f$ are assumed to be costly or inaccessible, and the optimizer only has access to a fixed offline dataset $S = \{(\vx_i, y_i)\}_{i=1}^m$, where each observation is generated by $y_i = f(\vx_i) + \epsilon_i$ with a noise $\epsilon_i$. The goal of offline MBO is to search an approximate maximizer $\vx^* \in \arg\max_{\vx \in \mathcal{X}} f(\vx)$ without any further interaction with the true function $f$. Because $f$ cannot be queried, the optimization must rely entirely on a learned surrogate model $h_\theta$ that approximates $f$ based on $\mathcal{D}$. 

\paragraph{Surrogate learning:}
In the first stage of offline MBO, a surrogate model is trained to approximate the unknown objective using only the offline dataset
$
S=\{(\vx_i,y_i)\}_{i=1}^m
$.
A hypothesis class $\mathcal{H}$ is selected, and a parametric model $h_\theta \in \mathcal{H}$ is fitted by minimizing an empirical surrogate loss:
\begin{equation*}
	\theta^*
	=
	\arg\min_{\theta\in\Theta}
	\ \mathcal{L}_S(h_\theta)
	+
	\Omega(\theta),
	\label{eq:surrogate_training}
\end{equation*}
where $\Omega(\theta)$ controls model complexity and $\mathcal{L}_S(h_\theta)$ denotes the empirical surrogate loss on $S$.
In this work, we consider two standard choices for $\mathcal{L}_S(\cdot)$.

(i) Mean squared error loss:
\begin{equation*}
	\mathcal{L}^{\mathrm{mse}}_S(h_\theta)
	=
	\frac{1}{m}
	\sum_{i=1}^m
	\ell_{\mathrm{mse}}
	\!\left(h_\theta(\vx_i),y_i\right),
	\label{eq:mse_empirical}
\end{equation*}
with the pointwise loss
\begin{equation*}
	\ell_{\mathrm{mse}}
	\!\left(h_\theta(\vx),y\right)
	=
	\big(h_\theta(\vx)-y\big)^2.
	\label{eq:mse_loss}
\end{equation*}

(ii) Pairwise ranking loss:

Let $\mathcal{P}_S:=\big\{(i,j)\in[m]^2:\ y_i>y_j\big\}$ denote the set of ranked index pairs induced by the dataset.
The empirical ranking loss is defined as
\begin{equation*}
	\mathcal{L}^{\mathrm{rank}}_S(h_\theta)
	=
	\frac{1}{|\mathcal{P}_S|}
	\sum_{(i,j)\in \mathcal{P}_S}
	\ell_{\mathrm{rank}}
	\!\left(h_\theta;\vx_i,\vx_j\right).
	\label{eq:rank_empirical}
\end{equation*}
A standard 0--1 pairwise loss is
\begin{equation*}
	\ell_{\mathrm{rank}}
	\!\left(h_\theta;\vx,\vx'\right)
	=
	\mathbf{1}
	\!\left\{
	h_\theta(\vx)
	\le
	h_\theta(\vx')
	\right\},
	\label{eq:pairwise_loss}
\end{equation*}
and a commonly used smooth surrogate takes the form
\begin{equation*}
	\tilde\ell_{\mathrm{rank}}
	\!\left(h_\theta;\vx,\vx'\right)
	=
	\phi
	\!\left(
	h_\theta(\vx')
	-
	h_\theta(\vx)
	\right),
	\label{eq:smooth_pairwise_loss}
\end{equation*}
where $\phi$ is a monotone margin function. The surrogate $h_{\theta^*}$ serves as an available proxy for the unknown objective.

\paragraph{Optimization:}
In the second stage of offline MBO, the learned surrogate $h_{\theta^*}$ serves as a tractable stand-in for the true objective, and the goal is to identify a design
\begin{equation*}
    \hat{\vx}_{\mathrm{MBO}}
    =
    \argmax_{\vx\in\mathcal{X}}
    h_{\theta^*}(\vx),
\end{equation*}
using an optimization procedure $\mathcal{A}(h_{\theta^*},\vx_0)$ initialized at some initial points $\vx_0$. When $h_{\theta^*}$ is differentiable, gradient-based methods are commonly applied. A typical update takes the projected gradient-ascent form
\begin{equation*}
    \vx_{t+1}
    =
    \Pi_{\mathcal{X}}\!\left(\vx_t + \eta_t \nabla_\vx h_{\theta^*}(\vx_t)\right),
    \  t=0,1,\dots,T,
\end{equation*}
where $\eta_t$ is a step size and $\Pi_{\mathcal{X}}$ represents a projection onto the feasible domain.

\section{Theoretical Results}\label{sec:theory}

%% 此处要有一段summary。介绍本section
This section studies the learnability of offline MBO from an optimization-oriented perspective.
We show that pairwise ranking losses yield strictly tighter guarantees than standard regression losses.
Our analysis further identifies distributional mismatch between the offline data and the near-optimal region as the dominant source of error, and characterizes an intrinsic limitation of offline MBO in terms of the geometric separation between near-optimal designs and the data manifold. Detailed proofs are in Appendix~\ref{sec:app-proof}.

\subsection{Learnability of Rank-based Offline MBO}

%% 说明MBO与回归任务的本质区别，因此要先定义一些理论需要的概念：MBO的泛化误差，Rademacher复杂度
% 与旨在恢复整个定义域上未知函数$f$精确逐点近似的标准监督回归不同，离线MBO具有根本不同的目标：优化器寻求具有高真实值$f(\vx^*)$的设计$\vx^*$，因此主要通过代理模型$h_{\theta^*}$的排序或保序行为来依赖它~\citep{tan2024offline}。换句话说，精确的全局预测对于成功优化既非必要也不充分；重要的是代理模型能正确区分优质设计与劣质设计，尤其是在与优化相关的区域。这一区别促使我们从经典回归误差指标转向以优化为导向的度量，这些度量量化了代理模型保持设计相对排序的能力。在本节中，我们形式化了这种基于排序的误差，引入了定义它的成对分布，并阐述了支持我们泛化分析的几个学习理论标准概念。
Unlike standard supervised regression, which aims to recover an accurate pointwise approximation of the unknown function $f$ over the entire domain, offline MBO has a fundamentally different objective: the optimizer seeks a design $\vx^*$ with high true value $f(\vx^*)$, and thus depends on the surrogate $h_{\theta^*}$ primarily through its ranking or order-preserving behavior~\citep{tan2024offline}. In other words, what matters is that the surrogate correctly distinguishes high-quality designs from inferior ones, especially in the regions relevant for optimization. This distinction motivates a shift from classical regression error metrics toward optimization-oriented measures that quantify the surrogate's ability to preserve the relative ranking of designs. In this section, we formalize such a ranking-based error, introduce the pairwise distributions that define it, and state several standard notions from learning theory that will support our generalization analysis.

\begin{definition}
For any tolerance level $\varepsilon \ge 0$, define the $\varepsilon$-near-optimal set
\begin{equation*}
    \mathcal{X}_\varepsilon
    :=
    \{\vx \in \mathcal{X} : f(\vx^*) - f(\vx) \le \varepsilon\}
\end{equation*}
and its complement $\mathcal{X}_{>\varepsilon} := \mathcal{X}\setminus\mathcal{X}_\varepsilon$.
We denote by $\rho_\varepsilon$ and $\rho_{>\varepsilon}$ the conditional distributions of designs in $\mathcal{X}_\varepsilon$ and $\mathcal{X}_{>\varepsilon}$, respectively.
\end{definition}

The theoretical error of interest evaluates whether the surrogate $h_\theta$ preserves the ranking between near-optimal and clearly suboptimal designs. Therefore, we propose the following definition of optimization-oriented ranking error.
% 为此，我们给出以下optimization-oriented排序误差的定义

\begin{definition}
Let the target pair distribution be
\begin{equation*}
    Q_\varepsilon(\vx,\vx') := \rho_\varepsilon(\vx)\,\rho_{>\varepsilon}(\vx'),
\end{equation*}
and the ranking error of a surrogate $h_\theta$ is defined as
\begin{equation*}
    \mathcal{E}^{\mathrm{rank}}_\varepsilon(\theta)
    :=
    \mathbb{E}_{(\vx,\vx')\sim Q_\varepsilon}
    \big[\ell_{\mathrm{rank}}(h_\theta;\vx,\vx')\big],
\end{equation*}
which measures the probability that the surrogate incorrectly ranks a near-optimal design below a strictly worse one.
\end{definition}

To analyze how minimizing empirical losses affects $\mathcal{E}^{\mathrm{rank}}_\varepsilon$, we rely on standard tools from statistical learning theory. In particular, uniform convergence arguments are expressed through the Rademacher complexity of a hypothesis class.

\begin{definition}
Let $\mathcal{F}$ be a class of real-valued functions on a domain $\mathcal{Z}$, and let $Z=\{z_1,\dots,z_m\}$ be a sample drawn from $\mathcal{Z}$. The empirical Rademacher complexity of $\mathcal{F}$ with respect to $Z$ is
\begin{equation*}
    \hat{\mathfrak{R}}_Z(\mathcal{F})
    :=
    \mathbb{E}_\sigma
    \left[
        \sup_{f\in\mathcal{F}}
        \frac{1}{m}\sum_{i=1}^m
        \sigma_i f(z_i)
    \right],
\end{equation*}
where $\sigma_1,\dots,\sigma_m$ are independent Rademacher variables taking values in $\{-1,+1\}$. This quantity captures the expressive capacity of $\mathcal{F}$ relative to the observed sample.
\end{definition}

% 排序误差 $\mathcal{E}^{\mathrm{rank}}_\varepsilon(\theta)$ 与上述复杂度概念将构成我们泛化分析的基础。在后续部分中，我们将训练过程中使用的经验替代损失与面向优化的误差联系起来，并描述分布偏移和OOD现象如何影响这种关系。
The ranking error $\mathcal{E}^{\mathrm{rank}}_\varepsilon(\theta)$ together with the above complexity notions will form the basis of our generalization analysis. In the following parts, we connect the empirical surrogate losses used during training to the optimization-oriented error, and characterize how distribution shift and OOD phenomena affect this relationship.

% 我们现在建立第一个主要理论结果：一个泛化界，将代理训练期间使用的经验性成对排序损失与前一节中引入的面向优化的排序误差$\mathcal{E}^{\mathrm{rank}}_\varepsilon(\theta)$联系起来。为准备分析，我们首先形式化适用一致收敛性论证的假设条件，随后给出连接经验风险与总体风险的技术性引理。这些要素将自然导向主要定理的证明。
We now establish the first main theoretical result: a generalization bound connecting the empirical pairwise ranking loss used during surrogate training to the optimization-oriented ranking error $\mathcal{E}^{\mathrm{rank}}_\varepsilon(\theta)$ introduced in the previous section. To prepare the analysis, we begin by formalizing the assumptions under which uniform convergence arguments apply, followed by technical lemmas relating the empirical and population risks. These ingredients will lead naturally into the main theorem.

% 回想一下，用于排序损失的训练对 $(\vx,\vx')$ 不需要遵循目标对分布$Q_\varepsilon$。设 $Q_{\mathrm{tr}}$ 表示训练对的数据生成分布。我们考虑所有有序对的集合。
Recall that training pairs $(\vx,\vx')$ used for the ranking loss need not follow the target pair distribution $Q_\varepsilon$. Let $Q_{\mathrm{tr}}$ denote the data-generating distribution for training pairs. We consider the full set of ranked pairs
$\mathcal{P}_S=\{(i,j)\in[m]^2:\ y_i>y_j\}$
induced by the dataset $S$.
The corresponding empirical ranking loss is
\begin{equation}
	\widehat{\mathcal{E}}_{\mathrm{tr}}^{\mathrm{rank}}(\theta)
	:=
	\frac{1}{|\mathcal{P}_S|}
	\sum_{(i,j)\in\mathcal{P}_S}
	\ell_{\mathrm{rank}}
	\left(
	h_\theta;
	\vx_i,
	\vx_j
	\right),
	(\vx_i,\vx_j)
	\sim
	Q_{\mathrm{tr}},
	\label{eq:empirical_rank_tr}
\end{equation}
where $|\mathcal{P}_S|$ denotes the number of admissible training pairs.
The theoretical objective of interest remains the population risk under the target pair distribution
\begin{equation*}
	\mathcal{E}^{\mathrm{rank}}_\varepsilon(\theta)
	:=
	\mathbb{E}_{(\vx,\vx')\sim Q_\varepsilon}
	\big[
	\ell_{\mathrm{rank}}
	\!\left(
	h_\theta;
	\vx,
	\vx'
	\right)
	\big].
	\label{eq:population_rank}
\end{equation*}
Bridging the gap between these two distributions requires assumptions that ensure the regularity of hypotheses and control of distributional mismatch.

\begin{assumption}
\label{assump:bounded}
For every $\theta \in \Theta$, the pairwise loss satisfies
\begin{equation*}
    0 \le \ell_{\mathrm{rank}}(h_\theta;\vx,\vx') \le 1,
    \ \forall (\vx,\vx')\in\mathcal{X}\times\mathcal{X}.
\end{equation*}
\end{assumption}
This simply ensures that the loss class is uniformly bounded,
a standard requirement for concentration inequalities and
Rademacher complexity arguments.

Then, we define the discrepancy between the target and training pair distributions as
\begin{equation*}
    \Delta_\varepsilon(\Theta)
    :=
    \sup_{\theta\in\Theta}
    \left|
        \mathbb{E}_{Q_\varepsilon}
        \ell_{\mathrm{rank}}(h_\theta;\vx,\vx')
        -
        \mathbb{E}_{Q_{\mathrm{tr}}}
        \ell_{\mathrm{rank}}(h_\theta;\vx,\vx')
    \right|.
\end{equation*}

% 术语$\Delta_\varepsilon(\Theta)$明确捕捉了训练对分布与面向优化的目标分布之间的不匹配。后续章节将使用几何概念（例如Wasserstein距离）来细化和上界这个量，但在第一个定理中我们抽象地处理它。
The term $\Delta_\varepsilon(\Theta)$ explicitly captures the mismatch between the training pair distribution and the optimization-oriented target distribution. Later sections will refine and upper-bound this quantity using geometric notions, e.g., Wasserstein distance, but for the first theorem we treat it abstractly.

Firstly, we state a uniform convergence lemma that relates the risk under $Q_{\mathrm{tr}}$ to its corresponding empirical estimator.

\begin{lemma}
\label{lemma:uniform_convergence}
Let $\{(\vx_i,\vx_j')| (i,j)\in [m]^2\}$ be training pairs induced by the dataset $S$ draw from the training distribution $Q_{\mathrm{tr}}$ and
$
    \mathcal{F}
    :=
    \bigl\{
        (\vx,\vx') \mapsto \ell_{\mathrm{rank}}(h_\theta;\vx,\vx')
        : \theta \in \Theta
    \bigr\}
$
denote the rank-based loss class. Under Assumptions~\ref{assump:bounded}, with probability at least $1-\delta$, we have
\begin{equation*}
    \sup_{\theta\in\Theta}
    \left|
        \mathbb{E}_{Q_{\mathrm{tr}}}
        \ell_{\mathrm{rank}}(h_\theta)
        -
        \widehat{\mathcal{E}}^{\mathrm{rank}}_{\mathrm{tr}}(\theta)
    \right|
    \le
    2\,\hat{\mathfrak{R}}_Z(\mathcal{F})
    +
    3\sqrt{\frac{\log(1/\delta)}{2m}}.
\end{equation*}
\end{lemma}

% 这是经验过程理论中的一个标准结果，直接适用于成对损失类 $\mathcal{F}$ 和训练分布 $Q_{\mathrm{tr}}$。现在，我们将引理~\ref{lemma:uniform_convergence} 与假设~\ref{assump:shift} 中的分布偏移分解相结合，以获得主要的泛化界。
This is a standard result from empirical process theory, directly adapted to the pairwise loss class $\mathcal{F}$ and the training distribution $Q_{\mathrm{tr}}$.
We are now ready to combine Lemma~\ref{lemma:uniform_convergence} with the distribution-shift decomposition from Assumption~\ref{assump:bounded} to obtain the main generalization bound.

\begin{theorem}
\label{thm:generalization}
Under Assumptions~\ref{assump:bounded}, with probability at least $1-\delta$ over the draw of training pairs from $Q_{\mathrm{tr}}$, for any $\theta\in\Theta$ satisfies
\begin{equation*}
    \mathcal{E}^{\mathrm{rank}}_\varepsilon(\theta)
    \le
    \widehat{\mathcal{E}}^{\mathrm{rank}}_{\mathrm{tr}}(\theta)
    +
    2\,\hat{\mathfrak{R}}_Z(\mathcal{F})
    +
    3\sqrt{\frac{\log(1/\delta)}{2m}}
    +
    \Delta_\varepsilon(\Theta).
\end{equation*}
\end{theorem}

%定理~\ref{thm:generalization}表明，以优化为导向的排序误差由四个部分控制：
%(i) 训练期间观察到的经验成对排序损失，
%(ii) 代理假设类的复杂度，
%(iii) 由于训练对数量有限导致的统计偏差，
%以及 (iv) 分布偏移项 $\Delta_\varepsilon(\Theta)$。
%前三个项反映了监督学习中的经典结果，而$\Delta_\varepsilon(\Theta)$ 是离线 MBO 所特有的：它量化了训练对分布与理论上识别接近最优设计所需分布之间的差异。后续部分将细化这一项，并将其明确地与 OOD 现象联系起来。
\begin{remark}
Theorem~\ref{thm:generalization} shows that the optimization-oriented ranking error is controlled by four components:
(i) the empirical pairwise ranking loss observed during training,
(ii) the complexity of the surrogate hypothesis class,
(iii) the statistical deviation due to a finite number of training pairs,
and (iv) the distribution-shift term $\Delta_\varepsilon(\Theta)$.
The first three terms mirror classical results in supervised learning, whereas $\Delta_\varepsilon(\Theta)$ is unique to offline MBO: It quantifies how much the training pair distribution differs from what is theoretically needed for identifying near-optimal designs. Subsequent sections will refine this term and relate it explicitly to OOD phenomena.
\end{remark}

\subsection{Ranking is Better than Regression}
% 上述结果确立了当使用成对排序损失训练替代模型时，面向优化的排序误差$\mathcal{E}^{\mathrm{rank}}_\varepsilon(\theta)$的泛化界。一个自然的问题是，这与使用标准逐点回归损失（如均方误差）训练替代模型相比如何。在本小节中，我们提供了正式的比较，并确定了在哪些条件下成对排序损失能比均方误差对$\mathcal{E}^{\mathrm{rank}}_\varepsilon(\theta)$产生严格更紧的界。
The above results establish a generalization bound for the optimization-oriented ranking error $\mathcal{E}^{\mathrm{rank}}_\varepsilon(\theta)$ when the surrogate is trained with a pairwise ranking loss. A natural question is how this compares to training the surrogate with a standard pointwise regression loss such as MSE. In this subsection, we provide a formal comparison.

% 为明确起见，我们将总体均方误差限定在接近最优和次优区域的定义如下：
For clarity, we define the population MSE restricted to the near-optimal and suboptimal regions:
\begin{align*}
	R^{\mathrm{mse}}_\varepsilon(\theta)
	&:=
	\mathbb{E}_{\vx\sim\rho_\varepsilon}
	\bigl[(h_\theta(\vx) - f(\vx))^2\bigr], \\
	R^{\mathrm{mse}}_{>\varepsilon}(\theta)
	&:=
	\mathbb{E}_{\vx\sim\rho_{>\varepsilon}}
	\bigl[(h_\theta(\vx) - f(\vx))^2\bigr].
\end{align*}

\begin{lemma}
	\label{lemma:mse-to-rank}
	Let $\gamma$ be a constant such that for almost every pair $(\vx,\vx')\sim Q_\varepsilon$ satisfies $f(\vx) - f(\vx') \ge \gamma$. For any $\theta\in\Theta$, we have
	\begin{equation*}
		\mathcal{E}^{\mathrm{rank}}_\varepsilon(\theta)
		\;\le\;
		\frac{4}{\gamma^2}
		\Big(
		R^{\mathrm{mse}}_\varepsilon(\theta)
		+
		R^{\mathrm{mse}}_{>\varepsilon}(\theta)
		\Big).
	\end{equation*}
\end{lemma}

We now combine the above reduction with the generalization guarantees for the MSE-based training.

\begin{theorem}
	\label{thm:rank-vs-mse}
	Let $S$ be the training set of size $m$, and $\mathcal{F}:=\{(\vx,y)\mapsto\ell_{\mathrm{mse}}(h_\theta(\vx),y)\}$ denote the MSE loss class, with probability at least $1-\delta$, we have
	\begin{align*}
		\mathcal{E}^{\mathrm{rank}}_\varepsilon(\theta)
		\le & 
		\frac{8}{\gamma^2}
		\Big(
		\widehat{R}_{S}^{\mathrm{mse}}(\theta)
		+
		2\hat{\mathfrak{R}}_Z(\mathcal{F})
		+
		3\sqrt{\frac{\log(1/\delta)}{2m}}
		\Big).
	\end{align*}
\end{theorem}

% 通过比较定理~\ref{thm:generalization}和定理~\ref{thm:rank-vs-mse}，我们可以发现在离线模型优化（MBO）中，当满足以下条件时，成对排序损失可证明优于均方误差（MSE）：
%(i) 小间隔γ导致近优设计与次优设计之间的区分度不足；
%(ii) 排序学习器在控制复杂度与适度Wasserstein偏移的情况下实现了较小的经验排序损失。
%直观而言，MSE将模型容量消耗在整个定义域的目标拟合上（包括对优化无意义的低价值区域），而排序损失直接聚焦于近优与次优设计之间的相对排序关系——这正是离线MBO所需的核心要素。
\begin{remark}
By comparing Theorem~\ref{thm:generalization} and~\ref{thm:rank-vs-mse}, we can find that the pairwise ranking loss provably outperforms MSE in offline MBO, when the following conditions are satisfied:
(i) The small margin $\gamma$ leads to poor separation between near-optimal and suboptimal designs, which is common in practice;
(ii) The ranking learner achieves a small empirical ranking loss with controlled complexity and moderate distribution shift.
Intuitively, MSE dedicates capacity to fitting the objective throughout the entire domain, including low-value regions that are irrelevant for optimization, whereas the ranking loss focuses directly on the relative ranking between different designs, which is precisely what offline MBO requires.
\end{remark}

\subsection{Distribution Shift and Algorithmic Inspiration}

% 定理~\ref{thm:generalization}中的泛化界将$\Delta_\varepsilon(\Theta)$确定为控制训练时代理损失与优化导向误差之间差距的关键项。然而，这个抽象的上确界并不能直接揭示训练数据与优化相关区域之间的分布关系。为了明确这种联系，我们对成对损失施加了温和的正则性条件，并推导了Wasserstein型上界，从而揭示出该误差界对$Q_\varepsilon$与数据分布之间分布外距离的依赖关系。
The generalization bound in Theorem~\ref{thm:generalization} identifies $\Delta_\varepsilon(\Theta)$ as the critical term governing the gap between the training-time surrogate loss and the optimization-oriented error. However, this abstract supremum does not directly reveal the distributional relationship between training data and the regions relevant for optimization. To make this connection explicit, we place mild regularity conditions on the pairwise loss and derive Wasserstein-type upper bounds that expose the dependence on distribution shift between $Q_\varepsilon$ and $Q_{\mathrm{tr}}$.

\begin{assumption}
\label{assump:lipschitz}
Let $\tilde\ell_{\mathrm{rank}}$ be a smooth surrogate of the ranking loss (e.g., hinge or logistic). Assume there exists $L>0$ such that for every $\theta\in\Theta$,
\begin{equation*}
    \big|
        \tilde\ell_{\mathrm{rank}}(h_\theta;\vx,\vx')
        -
        \tilde\ell_{\mathrm{rank}}(h_\theta;\vz,\vz')
    \big|
    \le
    L\, d_{\mathrm{pair}}\!\big((\vx,\vx'),(\vz,\vz')\big),
\end{equation*}
where the pairwise metric is
$
    d_{\mathrm{pair}}\big((\vx,\vx'),(\vz,\vz')\big)
    :=
    \|\vx - \vz\|_2 + \|\vx'-\vz'\|_2.
$
\end{assumption}

% 这个假设很自然：如果$h_\theta$满足利普希茨条件（这在有界范数神经网络中很常见），那么基于分数差异的平滑成对损失函数就会继承利普希茨连续性。该假设使得我们可以利用瓦瑟斯坦距离的坎托罗维奇-鲁宾斯坦对偶性。
This assumption is natural: If $h_\theta$ is Lipschitz, which is common for neural networks, then smooth pairwise losses based on score differences inherit Lipschitz continuity. The assumption enables the use of the Kantorovich--Rubinstein duality for Wasserstein distances~\citep{villani2008optimal}.

\begin{lemma}
	\label{lemma:wasserstein_bound}
	Under Assumption~\ref{assump:lipschitz}, for every $\theta\in\Theta$,
	\begin{equation*}
		\big|
		\mathbb{E}_{Q_\varepsilon}\tilde\ell_{\mathrm{rank}}(h_\theta)
		-
		\mathbb{E}_{Q_{\mathrm{tr}}}\tilde\ell_{\mathrm{rank}}(h_\theta)
		\big|
		\;\le\;
		L \, W_1(Q_\varepsilon, Q_{\mathrm{tr}}),
	\end{equation*}
	where $W_1$ denotes the 1-Wasserstein distance on the product space
	$(\mathcal{X}\times\mathcal{X}, d_{\mathrm{pair}})$.
\end{lemma}

This lemma turns the abstract discrepancy $\Delta_\varepsilon(\Theta)$ into a concrete geometric quantity that captures how much mass transportation is required to morph the training distribution $Q_{\mathrm{tr}}$ into the ideal distribution $Q_\varepsilon$.

We now derive a sharper characterization that decomposes $W_1(Q_\varepsilon,Q_{\mathrm{tr}})$ over individual coordinates, revealing how the OOD structure of the near-optimal design directly contributes to the overall error. Before that, we decompose the target and training pair distributions as
\begin{equation*}
    Q_\varepsilon = \rho_\varepsilon \otimes \rho_{>\varepsilon},
    \qquad
    Q_{\mathrm{tr}} = \mu_{\mathrm{tr}} \otimes \nu_{\mathrm{tr}},
\end{equation*}
where $\rho_\varepsilon$ and $\rho_{>\varepsilon}$ are the conditional distributions over $\mathcal{X}_\varepsilon$ and $\mathcal{X}_{>\varepsilon}$ respectively, and $\mu_{\mathrm{tr}},\nu_{\mathrm{tr}}$ are the marginals of $Q_{\mathrm{tr}}$ for the first and second coordinates.

\begin{corollary}
\label{corollary:marginal}
Under the assumptions of Theorem~\ref{thm:generalization} and Assumption~\ref{assump:lipschitz}, with probability at least
$1-\delta$,
\begin{equation*}
\begin{split}
    \mathcal{E}^{\mathrm{rank}}_\varepsilon(\theta)
    &\le
    \widehat{\mathcal{E}}^{\mathrm{rank}}_{\mathrm{tr}}(\theta)
    + 2\,\hat{\mathfrak{R}}_Z(\mathcal{F})
    + 3\sqrt{\frac{\log(1/\delta)}{2m}}
    \\&+ L\big(
        W_1(\rho_\varepsilon,\mu_{\mathrm{tr}})
        +
        W_1(\rho_{>\varepsilon},\nu_{\mathrm{tr}})
    \big).
\end{split}
\end{equation*}
\end{corollary}

% 这一分解表明，分布偏移的主要来源是$W_1(\rho_\varepsilon,\mu_{\mathrm{tr}})$，即近优区域与训练数据经验覆盖之间的Wasserstein距离。当优质设计远离数据流形时（离线MBO中的典型OOD场景），边际距离会变得很大，此时无论模型对观测数据对的拟合程度多好，代理排序误差都无法降低。
\begin{remark}
Corollary~\ref{corollary:marginal} shows that the dominant generalization error of ranking-based offline MBO is driven by the mismatch between the training marginal distribution $\mu_{tr}$ and the near-optimal distribution $\rho_\varepsilon$. Crucially, the leading term $W_1(\rho_\varepsilon,\mu_{tr})$ is not an intrinsic limitation of the hypothesis class, but depends on how the training data are constructed and utilized. This observation suggests that ranking performance can be improved by explicitly reshaping the effective training marginal toward near-optimal regions, rather than modifying the ranking objective itself. We instantiate this principle in Section~\ref{sec:method} by constructing an empirical approximation of $\rho_\varepsilon$ from high-quality samples, which leads to consistent empirical gains.
\end{remark}

% 在下一小节中，我们将进一步将$W_1(\rho_\varepsilon,\mu_{\mathrm{tr}})$与$\mathcal{X}_\varepsilon$到数据支撑集的平均距离联系起来，从而得出一个明确的OOD脆弱性几何度量。
% In the next subsection, we will further connect $W_1(\rho_\varepsilon,\mu_{\mathrm{tr}})$ to the average distance of $\mathcal{X}_\varepsilon$ from the data support, yielding an explicit geometric measure of OOD vulnerability.

\subsection{Geometric Separation of Offline MBO}

% 推论~\ref{corollary:marginal}中精炼的边际边界强调了$W_1(\rho_\varepsilon,\mu_{\mathrm{tr}})$作为主导项决定了离线MBO中的优化失败。为了更具体地理解这一量值，我们现在通过考察接近最优设计与数据流形之间的距离，将其与离线数据集的几何结构联系起来。
The refined marginal bound in Corollary~\ref{corollary:marginal} highlights $W_1(\rho_\varepsilon,\mu_{\mathrm{tr}})$ as the dominant term governing optimization failure in offline MBO. To understand this quantity more concretely, we now relate it to the geometric structure of the offline dataset by examining the distance between near-optimal designs and the data manifold.

% 由于$\mu_{\mathrm{tr}}$的支撑集集中在离线数据集或其紧邻区域（取决于配对采样机制），Wasserstein距离$W_1(\rho_\varepsilon,\mu_{\mathrm{tr}})$衡量了$\varepsilon$-近优集与代理模型接收训练信号的区域之间的距离。当近优区域$\mathcal{X}_\varepsilon$远离数据流形时，代理模型必须进行激进的外推，这不可避免地会导致如\citet{kim2025offline}所指出的分布外高估现象。本节将这一直观认识形式化。
Since $\mu_{\mathrm{tr}}$ is supported on the offline dataset or its immediate vicinity, the Wasserstein distance $W_1(\rho_\varepsilon,\mu_{\mathrm{tr}})$ measures how far the $\varepsilon$-near-optimal set lies from the region where the surrogate receives training signal. When the near-optimal region $\mathcal{X}_\varepsilon$ is far from the data manifold, the surrogate must extrapolate aggressively, inevitably producing OOD overestimation as noted by \citet{kim2025offline}. 

% We assume that $\mu_{\mathrm{tr}}$ is supported on $\mathcal{M}$ and that $\mathcal{M}$ is closed under projection. Surrogate training pairs are almost always sampled from or near the dataset $S$, and thus $\mu_{\mathrm{tr}}$ lives on a low-dimensional structure reflecting the distribution of observed designs.

\begin{corollary}
\label{corollary:geometry}
Let $\mathcal{M}\subseteq\mathcal{X}$ denote the support of $\mu_{\mathrm{tr}}$. For any $\vx \in \mathcal{X}$, define the distance to the data manifold as
$\mathrm{dist}(\vx,\mathcal{M})
    :=
    \inf_{\vz\in\mathcal{M}}\|\vx - \vz\|_2.$
Under the Assumption~\ref{assump:bounded} and~\ref{assump:lipschitz},
with probability at least $1-\delta$,
\begin{equation*}
\begin{split}
    \mathcal{E}^{\mathrm{rank}}_\varepsilon(\theta)
    &\le
    \widehat{\mathcal{E}}^{\mathrm{rank}}_{\mathrm{tr}}(\theta)
    + 2\,\hat{\mathfrak{R}}_Z(\mathcal{F})
    + 3\sqrt{\frac{\log(1/\delta)}{2m}}
    \\&+ L\!\left(
        \mathbb{E}_{\vx\sim\rho_\varepsilon}\!
        \mathrm{dist}(\vx,\mathcal{M})
        +
        C_{\mathrm{calib}}
        +
        W_1(\rho_{>\varepsilon},\nu_{\mathrm{tr}})
    \right),
\end{split}
\end{equation*}
where $C_{\mathrm{calib}}$ is a calibration term depending only on the internal spread of $\mathcal{M}$ under $\mu_{\mathrm{tr}}$.
\end{corollary}

% 推论~\ref{corollary:geometry} 明确给出了离线MBO问题中优化难度的几何特征描述。其核心项$\mathbb{E}_{\vx\sim\rho_\varepsilon}\mathrm{dist}(\vx,\mathcal{M})$量化了最优设计方案与数据流形之间的距离，该距离揭示了离线MBO问题固有的分布外脆弱性：当高质量设计方案无法被离线数据集中的任何点充分逼近时，代理模型将被迫进行深度外推，从而引发文献记载的分布外分数膨胀现象\citep{kim2025offline}——此时设计空间中未观测区域会获得虚高的预测分数，误导基于梯度的搜索陷入次优区域。
\begin{remark}
Corollary~\ref{corollary:geometry} offers a geometric characterization of optimization difficulty in offline MBO. Its central quantity, $\mathbb{E}_{\vx\sim\rho_\varepsilon}\mathrm{dist}(\vx,\mathcal{M})$, measures how far near-optimal designs lie from the data manifold. This distance captures the intrinsic OOD vulnerability of offline optimization: When high-quality designs are poorly supported by the offline data, surrogate models must extrapolate, leading to over-optimistic predictions in unobserved regions and misleading gradient-based search~\citep{kim2025offline}. Consequently, even with perfect training optimization and controlled model complexity, the optimization-oriented ranking error remains large unless the near-optimal region lies sufficiently close to the data manifold. This result formalizes the empirical observation that OOD structure fundamentally limits the effectiveness of offline MBO.
\end{remark}

\section{Distribution-Aware Ranking (DAR)}\label{sec:method}

\paragraph{Motivation.}
Corollary~\ref{corollary:marginal} identifies ranking as the appropriate learning objective for offline MBO and further reveals that the dominant source of optimization error arises from distributional mismatch between the training data and near-optimal designs.  
Importantly, this mismatch depends not only on the choice of surrogate objective, but also on how the training data are used to construct the surrogate model.
Motivated by this observation, we propose a simple yet effective ranking-based optimization method that explicitly reshapes the effective training data distribution toward near-optimal regions.

Rather than modifying the ranking objective, our approach operates at the level of data construction.
By selectively emphasizing high-quality samples during training, the induced marginal distribution of the ranking model is brought closer to the near-optimal distribution, thereby directly targeting the dominant error term identified in our theory. A full description of our algorithm is presented in Appendix~\ref{sec:app-alg}.

\paragraph{Dataset construction.}
Given an offline dataset $S = \{(\vx_i, y_i)\}_{i=1}^n$, where $y_i$ denotes the observed performance of design $\vx_i$, we first partition the dataset according to target values.
Specifically, let $q_\varepsilon$ denote the $\varepsilon$-quantile of $\{y_i\}_{i=1}^n$.
We define the near-optimal subset $S_\varepsilon = \{\vx_i \in S \mid y_i \ge q_\varepsilon\}$, and the suboptimal subset $S_{>\varepsilon} = S \setminus S_\varepsilon.$
This empirical partition approximates the theoretical sets $\mathcal{X}_\varepsilon$ and $\mathcal{X}_{>\varepsilon}$ introduced in Section~3, and serves as an estimate of the near-optimal and suboptimal regions, respectively.

Compared to approaches that penalize global ranking errors over the entire dataset~\citep{tan2024offline}, this construction focuses learning capacity on the ranking between near-optimal and clearly suboptimal designs, which is most relevant for downstream optimization.

\paragraph{Ranking loss learning.}
Training pairs $(\vx_1, \vx_2)$ are sampled to approximate the target ranking behavior.
The primary component of training pairs consists of $\vx_1 \in S_\varepsilon$ and $\vx_2 \in S_{>\varepsilon}$, which mimics the target pairwise distribution $Q_\varepsilon = \rho_\varepsilon \otimes \rho_{>\varepsilon}$ and directly targets the minimization of the ranking error $\mathcal{E}^{\mathrm{rank}}_\varepsilon$.

In addition, we introduce a regularization component with ratio $\lambda \in [0,1]$, where both $\vx_1, \vx_2 \in S_\varepsilon$.
This component encourages consistent ranking within the high-quality region and stabilizes optimization along the near-optimal manifold.
The overall pair sampling distribution is thus a mixture of cross-region and intra-region pairs.

The surrogate model $h_\theta$ is trained using the margin ranking loss with margin $\beta$, which serves as a smooth surrogate $\tilde{\ell}_{\mathrm{rank}}$.
For a pair where $\vx_1$ is preferred over $\vx_2$, the loss is given by
\begin{equation}
\mathcal{L}(\vx_1, \vx_2)
=
\max\!\left(0,\,
\beta - \big(h_\theta(\vx_1) - h_\theta(\vx_2)\big)
\right).
\end{equation}

\paragraph{Optimization adaptation.}
A practical challenge of ranking-based surrogate models is that the absolute scale of model predictions is not identifiable.
Unlike regression objectives, which encourage accurate prediction of target values, ranking losses are invariant to affine transformations of the output.
As a result, the scale of predictions produced by different ranking losses can vary substantially, which directly affects the magnitude of gradients used by gradient-based optimizers during the design search phase~\citep{tan2024offline}.

To ensure stable and fair optimization across different surrogate objectives, following~\citet{tan2024offline}, we apply an output adaptation step after training the ranking model.
Specifically, we first evaluate the trained surrogate $h_\theta$ on the entire training set and compute the empirical mean $\tilde{\mu}$ and standard deviation $\tilde{\sigma}$ of its predictions.
We then normalize the model output using a z-score transformation by
$\tilde{h}_\theta(\vx)
=
\frac{h_\theta(\vx) - \tilde{\mu}}{\tilde{\sigma}}$.
All gradient-based optimization steps are performed with respect to the normalized surrogate $\tilde{h}_\theta$.

This normalization aligns the scale of gradients with those of regression-based surrogates, allowing us to reuse standard hyper-parameter settings for the step size $\eta$ and the number of optimization steps $T$ as in prior work~\citep{tan2024offline}.

\section{Experiments}

%%%%%%%%%% 大纲 %%%%%%%%%%
% 1.实验设置：benchmark任务介绍、对比方法介绍、评估准则、超参设置
% 2.实验结果：
% 3.讨论和未来工作：用1-2句话凝练本研究最核心的发现。解释发现的意义，并与前人研究对比。挖掘发现背后的机制、原因或普遍意义。客观、坦诚地说明研究的局限及其来源（如样本、方法）。基于研究发现和局限性，提出具体、可行的未来研究方向。
% 4.结论：总结全文，重申研究的最核心价值与贡献。要回答的问题：“本研究最终解决了什么问题？”、“最关键的贡献是什么？”
%%%%%%%%%%%%%%%%%%%%%%%%%

% 在本节中，我们通过实验验证理论发现。在章节~\ref{sec:exp-branin}中，我们首先在Branin函数上比较基于回归的方法、标准排序方法以及我们的DAR方法，以验证排序目标确实能带来更优的离线优化性能，并且如推论~\ref{corollary:marginal}所预测，重塑训练数据分布能进一步提升结果。我们还研究了近优设计与训练数据之间几何间隔增大的影响，揭示了离线MBO方法的内在局限性——这正是推论~\ref{corollary:geometry}所描述的特性。在Design-Bench基准~\citep{trabucco2022design}上的扩展实验进一步验证了我们提出的分布感知排序方法的有效性，具体结果见章节~\ref{sec:exp-db}。
In this section, we empirically validate the theoretical findings. In Section~\ref{sec:exp-branin}, we first compare regression-based methods, standard ranking approaches, and our DAR on the Branin function to verify that ranking objectives lead to superior offline optimization performance and that reshaping the training data distribution further improves results, as predicted by Corollary~\ref{corollary:marginal}. We also investigate the effect of increasing geometric separation between near-optimal designs and the training data, illustrating the intrinsic limitation of offline MBO characterized by Corollary~\ref{corollary:geometry}.
Extended experiments on the Design-Bench~\citep{trabucco2022design} further validate the effectiveness of our distribution-aware ranking method, as illustrated in Section~\ref{sec:exp-db}.

\begin{figure*}[!ht]
    \centering
     %第一个子图 (a)
     \captionsetup[subfigure]{skip=0pt, font=normalfont} 
     \begin{subfigure}[b]{0.32\linewidth}
         \centering
         \includegraphics[width=\linewidth]{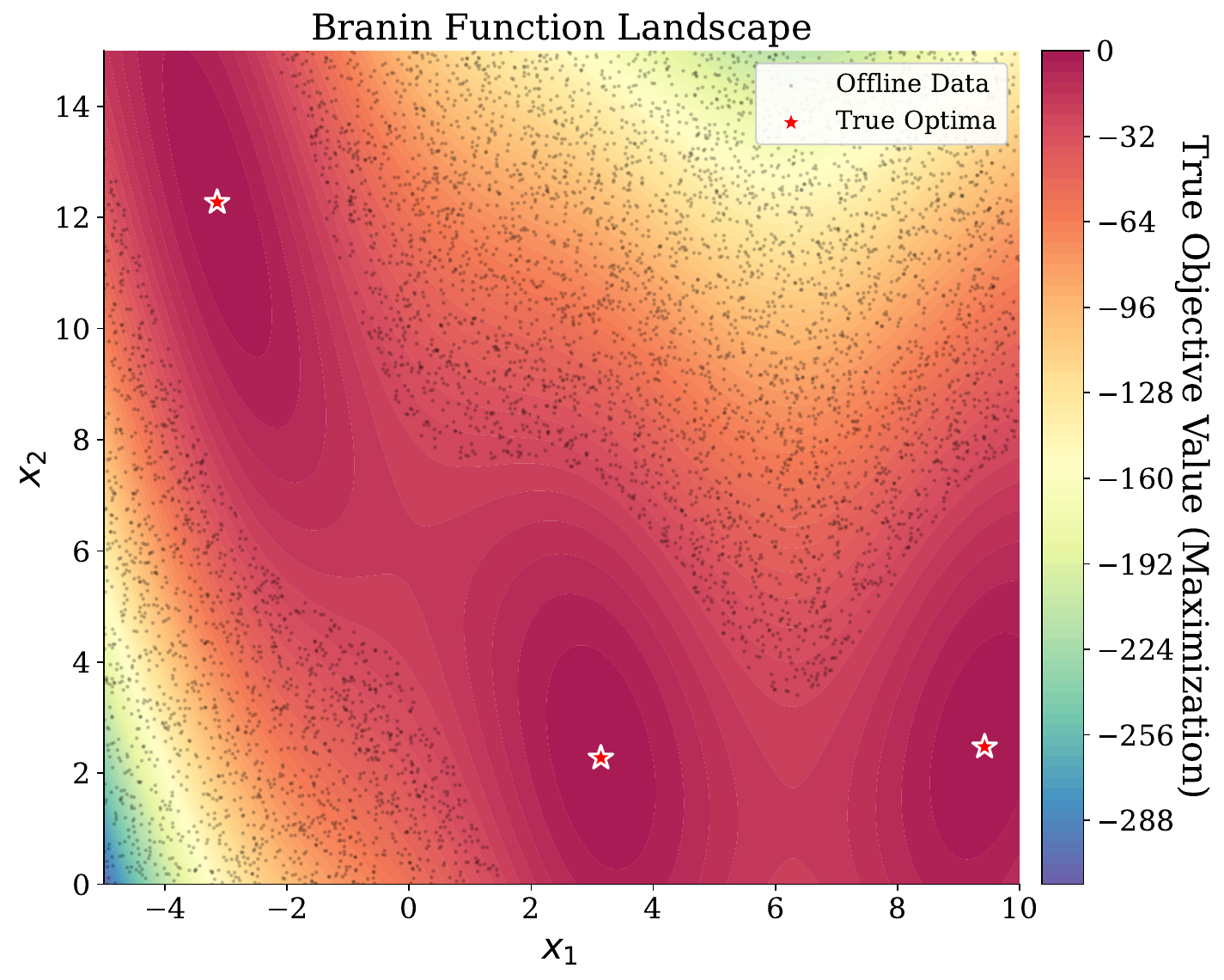}
         \caption{} % 空的大括号会自动生成 (a)
         \label{fig:branin_a}
     \end{subfigure}
     \hfill % 添加一点水平间距
     %第二个子图 (b)
    \begin{subfigure}[b]{0.32\linewidth}
        \centering
        \includegraphics[width=\linewidth]{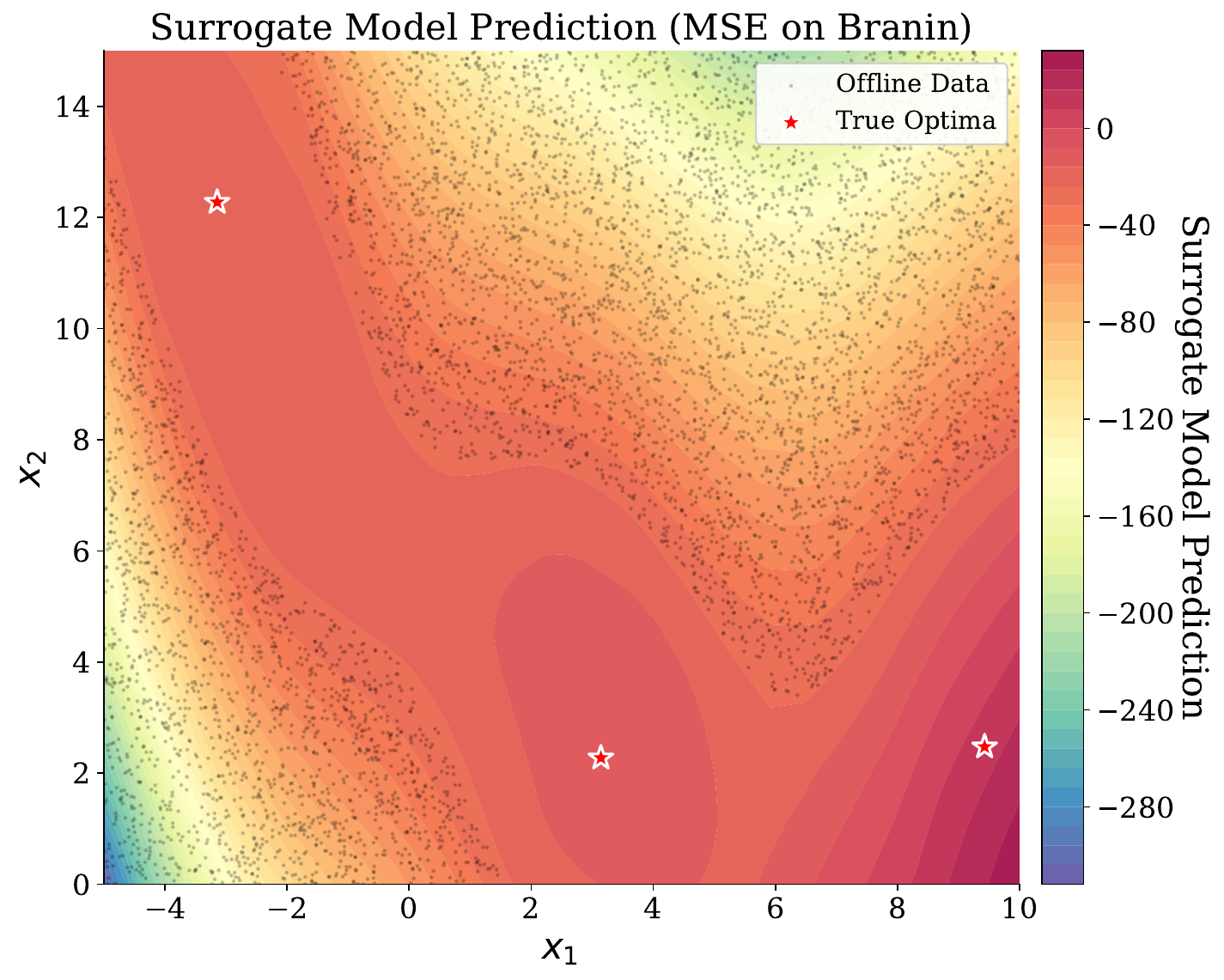}
        \caption{} % 空的大括号会自动生成 (b)
        \label{fig:branin_b}
    \end{subfigure}
    \hfill % 添加一点水平间距
    % 第三个子图 (c)
    \begin{subfigure}[b]{0.32\linewidth}
        \centering
        \includegraphics[width=\linewidth]{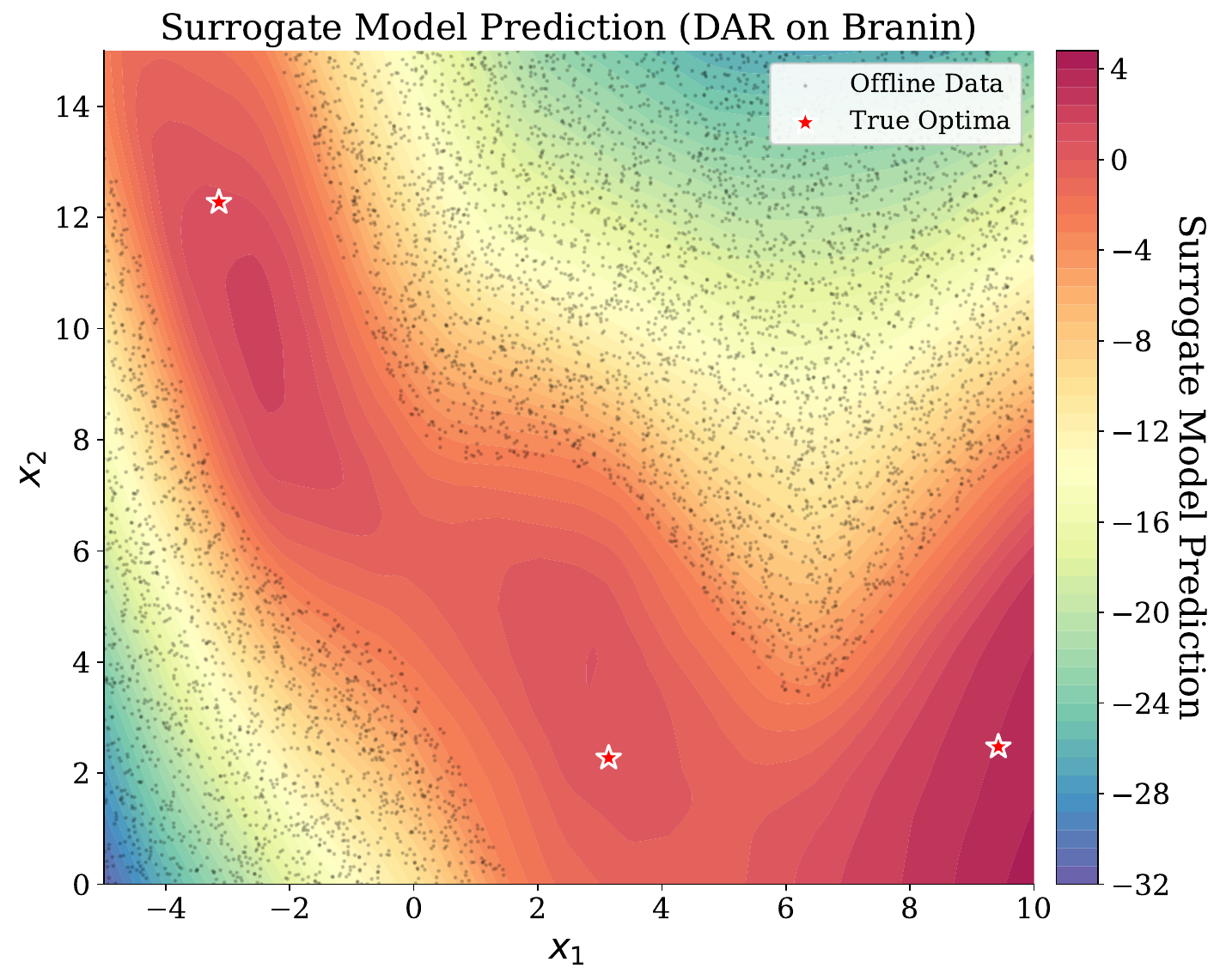}
        \caption{} % 空的大括号会自动生成 (c)
        \label{fig:branin_c}
    \end{subfigure}
    
    \vspace{-0.5em}
    \caption{
    Landscape analysis on the Branin function. 
    \textbf{(a)} The landscape of the Branin function and the distribution of the offline dataset (black dots), which consists of the worst 60\% of designs. 
    \textbf{(b)} and \textbf{(c)} present the prediction landscape derived from the MSE-trained surrogate and DAR-trained surrogate, respectively. 
    The DAR surrogate accurately extrapolates the landscape structure, recovering the three distinct peaks of the true optima with high fidelity, whereas the MSE baseline fails to precisely  recover these modes.
    % \textbf{(b)} The relationship between ranking error probability and the manifold radius $\delta$ with a fixed $\epsilon$~=~5\%. The plot compares the expected ranking error, the generalization gap, and the empirical ranking error. 
    % \textbf{(c)} Objective value performance versus manifold radius $\delta$. The shaded area represents the approximation gap between the true ground truth values and the surrogate estimated values.
    }
    \label{fig:branin}
    \vspace{-1em}
\end{figure*}

\begin{figure}
    \centering
    \includegraphics[width=0.7\linewidth]{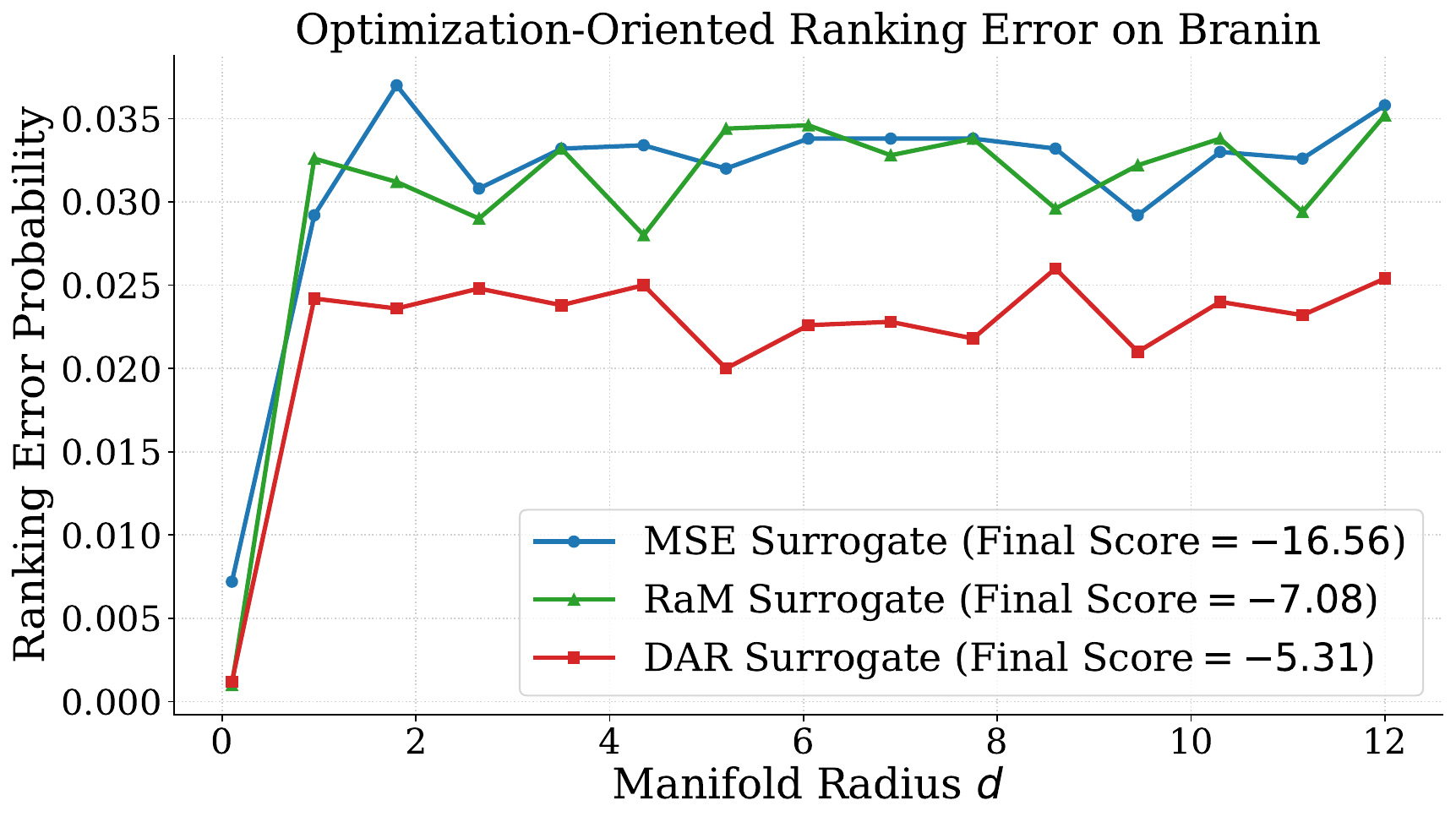}
    \caption{Comparison of surrogates trained with MSE, RaM, and DAR on optimization-oriented ranking error.}
    \label{fig:ranking-error}
    \vspace{-1.5em}
\end{figure}

\subsection{Analysis on Branin Function}
\label{sec:exp-branin}
% 在图~\ref{fig:branin}中，我们展示了Branin函数上离线MBO的可视化示意图。按照将表现最差的收集点作为离线训练数据集的通用协议~\citep{trabucco2022design}，我们使用表现最差的60%设计以均方误差损失训练代理模型。
In Figure~\ref{fig:branin}, we additionally provide a visualized illustration of offline MBO on the Branin function. Following the common protocol of treating the worst-performing collected points as the offline training dataset~\citep{trabucco2022design}, we train the surrogate model using the worst 60\% designs to create a shift between the training distribution and the true global optima.
Figure~\ref{fig:branin_b} and~\ref{fig:branin_c} deliver the predictive landscapes of surrogates trained with MSE and our proposed method, DAR, respectively. 
From the figures, we can observe that while the MSE-trained surrogate yields a smoothed and flattened landscape that fails to identify the high-scoring regions (especially missing the top left region), DAR successfully reconstructs the multi-modal topography of the Branin function. Notably, DAR sharply recovers all three distinct peaks associated with the true optima, demonstrating better capability in extrapolating to unseen promising regions.
% 在图~\ref{fig:branin_b}中，我们可以观察到，随着流形半径$\delta$的增大，期望排序误差$\mathcal{E}^{\mathrm{rank}}_\varepsilon(\theta)$呈现出持续上升的趋势，这证实了低训练误差并不能保证在OOD区域获得准确的排序。因此，如~\ref{fig:branin_c}所示，优化代理模型预测结果的模型内部搜索过程，必然会陷入模型给出高分的低分OOD区域。
% In Figure~\ref{fig:branin_b}, we can observe that as the manifold radius $\delta$ increases, the expected ranking error $\mathcal{E}^{\mathrm{rank}}_\varepsilon(\theta)$ shows a consistent increasing trend, confirming that low training error does not guarantee accurate ranking in OOD regions. As a result, the model-inner search procedure, which optimizes the surrogate model's prediction, would inevitably dive into low-scoring OOD regions where the model assigns high scores, as shown in Figure~\ref{fig:branin_c}.  

To better understand the performance of different surrogate models, we quantify the optimization-oriented ranking error $\mathcal{E}_{\varepsilon}^{\textrm{rank}}$ against the radius to the data manifold $d$. 
We set $\varepsilon=5\%$ to construct $S_\varepsilon$ and $S_{>\varepsilon}^ d = \{ \vx\mid \operatorname{dist}(\vx,\mathcal{M})\le d \}$. 
Then we calculate the ranking error probability between $S_\varepsilon$ and $S_{>\varepsilon}$ to approximate the expected global optimization-oriented ranking error.
We access the performance of the surrogate model trained with MSE, RaM~\citep{tan2024offline}~\footnote{As our training objective (i.e., Eq.~\eqref{eq:smooth_pairwise_loss}) is a pairwise ranking loss, to ensure a fair comparison, we set RankCosine~\citep{rankcosine}, the best-performing pairwise ranking loss reported in~\citet{tan2024offline}, as the default ranking loss for RaM.}, and DAR.
As shown in Figure~\ref{fig:ranking-error}, we find that our method, DAR, consistently achieves the lowest ranking error compared to the MSE and RaM across varying $ d$, and the lower ranking error of DAR directly correlates with superior final optimization scores compared to the baselines.
Notably, we additionally observe a consistent trend where the ranking error increases as the radius $ d$ grows, highlighting the intrinsic limitation of offline MBO demonstrated by Corollary~\ref{corollary:geometry}, i.e., inevitably poor optimization-oriented generalization in the out-of-distribution regions. 

\begin{table*}[!ht]
	\caption{100th percentile normalized score in Design-Bench, where the best and runner-up results on each task are \textbf{\textcolor{blue}{Blue}} and \textbf{\textcolor{violet}{Violet}}. $S$(best) denotes the best score in the offline dataset.
		Results of compared methods are either inherited from~\citet{tan2024offline} or run by us.
	}%\vspace{0.3em}
	\centering
	\resizebox{\linewidth}{!}{
 \begin{tabular}{l|ccccc|c}
			\toprule
			Method        & Ant                                    & D'Kitty                                & Superconductor                         & TF-Bind-8                              & TF-Bind-10                             & Avg. Rank                                      \\  \midrule 
			$S$(best) & 0.565 & 0.884 & 0.400 & 0.439 & 0.467 & /  \\ \midrule
			PGS~\citep{pgs}               & 0.715 ± 0.046                          & 0.954 ± 0.022                          & 0.444 ± 0.020                          & 0.889 ± 0.061                          & 0.634 ± 0.040                          & 5.4                              \\
			FGM~\citep{fgm}               & 0.923 ± 0.023                          & 0.944 ± 0.014                          & 0.481 ± 0.024                          & 0.811 ± 0.079                          & 0.611 ± 0.008                          & 5.8                              \\
			Match-OPT~\citep{hoang2024learning}         & 0.933 ± 0.016                          & 0.952 ± 0.008                          & 0.504 ± 0.021                          & 0.824 ± 0.067                          & 0.655 ± 0.050                          & 4.0                               \\
			GTG~\citep{gtg}               & 0.855 ± 0.044                          & 0.942 ± 0.017                          & 0.480 ± 0.055                          & 0.910 ± 0.040                          & 0.619 ± 0.029                          & 5.6                              \\ 
			RaM~\citep{tan2024offline}    & 0.940 ± 0.028                          & 0.951 ± 0.017                          & \textbf{\textcolor{blue}{0.514 ± 0.026}}                          & \textbf{\textcolor{violet}{0.982 ± 0.012}} & \textbf{\textcolor{violet}{0.675 ± 0.049}}    & \textbf{\textcolor{violet}{2.6}} \\
            ROOT~\citep{dao2025root} & \textbf{\textcolor{blue}{0.958 ± 0.012}} &  \textbf{\textcolor{blue}{0.967 ± 0.001}} & 0.451 ± 0.032 & 0.977 ± 0.015 &  0.652 ± 0.020 & 3.0 \\  \midrule
			\textbf{DAR (Ours)} & \textbf{\textcolor{violet}{0.941 ± 0.013}}                          & \textbf{\textcolor{violet}{0.959 ± 0.006}}                          & \textbf{\textcolor{violet}{0.506 ± 0.008}}                          & \textbf{\textcolor{blue}{0.988 ± 0.008}} & \textbf{\textcolor{blue}{0.686 ± 0.056}}    & \textbf{\textcolor{blue}{1.6}} \\
			\bottomrule
	\end{tabular}}
 \label{tab:main-results}
 \vspace{-1em}
\end{table*}

\subsection{Experiments on Design-Bench}
\label{sec:exp-db}
In this subsection, we conduct experiments on Design-Bench~\citep{trabucco2022design}, a prevalent benchmark for offline MBO, showing the superiority of our DAR method. 
% We first introduce the experimental settings and then deliver the results in Section~\ref{sec:exp-results}. 
% 实验设置部分如果超页可以挪到附录中
% \subsubsection{Experimental Settings}
% \label{sec:exp-settings}

\paragraph{Benchmark tasks.}
We conduct experiments on five design tasks from the Design-Bench~\citep{trabucco2022design}, covering both continuous and discrete design tasks. The continuous tasks include \textbf{Ant Morphology}~\citep{gym}, \textbf{D'Kitty Morphology}~\citep{robel}, and \textbf{Superconductor}~\citep{superconductor}. Ant and D'Kitty aim to optimize the physical structures of an ant-like robot and a quadruped robot, respectively, by searching over 60- and 56-dimensional continuous parameter spaces to maximize locomotion speed in physics-based simulators, while the goal of Superconductor is to identify an 86-dimensional superconducting
material that maximizes the critical temperature. The discrete tasks consist of \textbf{TF-Bind-8} and \textbf{TF-Bind-10}~\citep{tfbind}, where the design variables are DNA sequences of length 8 and 10, and the objective is to identify sequences that achieve high binding affinity with a target transcription factor. Together, these tasks provide a diverse set of benchmarks for evaluating design optimization methods across high-dimensional continuous control problems and combinatorial sequence design scenarios.

\paragraph{Compared methods.} We compare representative methods since 2024, which encompasses forward surrogate based methods and recent generative approaches. 
The first category includes PGS~\citep{pgs}, FGM~\citep{fgm}, Match-OPT~\citep{hoang2024learning}, and RaM~\citep{tan2024offline}, where methods typically learn a surrogate and optimize inside it. 
The generative methods contain GTG~\citep{gtg} and ROOT~\citep{dao2025root}, which model the distribution of the design variables $\vx$ and sample the maximizer.

\paragraph{Implementation details and hyperparameters.} Following recent works in offline MBO~\citep{trabucco2021conservative,tan2024offline}, we instantiate the surrogate model $h_\theta$ as an MLP with two layers, each of which contains 2048 units. ReLU is employed as the activation function. 
For each baseline, we adopt the optimized settings from the original papers.
For DAR, we choose the quantile $\varepsilon=0.2$, the ratio of high-quality training pairs $\lambda=0.1$, and the ranking margin $\beta=0.4$. We provide a detailed ablation results of these hyperparameters in Appendix~\ref{appendix:additional-exp}. 

\paragraph{Evaluation and metrics.} 
We follow the evaluation protocol of Design-Bench~\citep{trabucco2022design} to attach the performance of all compared methods. 
Specifically, we collect the 128 solutions proposed by the algorithm and report the 100th percentile normalized ground-truth score.
The score is normalized by $\frac{y-y_{\min}}{y_{\max}-y_{\min}}$, where $y_{\min}$ and $y_{\max}$ represent the lowest and highest scores in the full unseen dataset. 

% \subsubsection{Experimental Results}
% \label{sec:exp-results}
\paragraph{Experimental results.}
Table~\ref{tab:main-results} delivers the results of out experiments. Our method, DAR, achieves the best average rank of 1.6, while the second and third best methods, RaM and ROOT, obtain average ranks of 2.6 and 3.0, respectively.
Among all methods, DAR performs best on two discrete tasks, TF-Bind-8 and TF-Bind-10, and is runner-up on other three continuous tasks, Ant, D'Kitty, and Superconductor, showing the clear effectiveness of our proposed method.

\section{Discussion and Future Work}
In this section, we discuss the broader implications of our theoretical findings and outline several future directions.

\paragraph{Optimization dynamics.}
% 我们的分析聚焦于在接近最优和次优设计之间正确排序的可学习性，并隐含假设优化阶段能够可靠地找到代理模型的最大值。然而在实践中，离线模型优化依赖于具体的搜索过程，这些过程的行为具有路径依赖性，并对代理模型的局部特性敏感。即使具备正确的全局排序，优化仍可能因平滑性不足或梯度不稳定而失败。将理论扩展以纳入优化动态（如轨迹依赖性保证或搜索稳定性）是一个重要的研究方向。
Our analysis focuses on the learnability of correct rankings between near-optimal and suboptimal designs, and implicitly assumes that the optimization stage can reliably recover the surrogate maximizer. In practice, however, offline MBO depends on concrete search procedures whose behavior is path-dependent and sensitive to local surrogate properties. Even with a correct global ranking, optimization may fail due to poor smoothness or unstable gradients. Extending the theory to incorporate optimization dynamics, such as trajectory-dependent guarantees or search stability, is an important direction.

\paragraph{Data reshaping strategies.}
% 我们通过一种简单的基于分位数的数据重塑方案来实例化我们的理论，该方案直接针对主要的分布不匹配项。然而，这仅代表了可能方法中的有限子集。其他策略，包括软重加权、几何感知对构建或近最优区域的生成增强，在理论上都有充分的依据且尚未被探索。理解不同的数据重塑机制如何影响面向优化的泛化，是未来工作的一个很有前景的方向。
We instantiate our theory with a simple quantile-based data reshaping scheme, which directly targets the dominant distributional mismatch term. However, this represents only a limited subset of possible approaches. Alternative strategies, including soft reweighting, geometry-aware pair construction, and generative/active augmentation of near-optimal regions, are theoretically well-motivated and remain unexplored. Understanding how different data reshaping mechanisms affect optimization-oriented generalization is a promising avenue for future work.

\paragraph{Conservative optimization.}
% 尽管我们描述了离线模型优化（MBO）由于几何分离而在本质上不可靠的内在机制，但我们并未研究替代的决策目标。在高风险场景中，保守策略（如带拒绝的学习、基于弃权的优化或考虑分布外数据的排序）可能比无约束的最大化更可取。为这类保守或风险感知的离线MBO变体开发理论框架是未来研究的重要方向。
While we characterize intrinsic regimes where offline MBO is fundamentally unreliable due to geometric separation, we do not study alternative decision objectives. In high-risk settings, conservative strategies such as learning with rejection, abstention-based optimization, or OOD-aware ranking may be preferable to unconstrained maximization. Developing theoretical frameworks for such conservative or risk-aware variants of offline MBO is an important direction for future research.

\section{Conclusion}

In this work, we study offline MBO from a learnability perspective and show that correct ranking among high-quality designs, rather than accurate value prediction, is the appropriate objective for offline MBO. We develop an optimization-oriented theoretical framework that establishes tighter generalization guarantees for ranking loss than regression loss, identifies distributional mismatch as the dominant source of optimization error, and characterizes a fundamental limitation arising from the separation between near-optimal designs and the training data. Together, these results clarify both the potential and the inherent limits of offline MBO, and provide principled guidance for the design and evaluation of future offline optimization methods.

\section*{Acknowledgement}
This work was supported by the National Natural Science Foundation of China (No. 62306104, 62441225 and 62572171), Basic Research Program of Jiangsu (No. BK20253011), Hong Kong Scholars Program (No. XJ2024010), Research Grants Council of the Hong Kong Special Administrative Region, China (GRF Project No. CityU11212524), Natural Science Foundation of Jiangsu Province (No. BK20230949), Jiangsu Association for Science and Technology (No. JSTJ2024285), China Postdoctoral Science Foundation (No. 2023TQ0104), and the High Performance Computing Platform of Hohai University.

% \bibliographystyle{alpha}
\bibliography{ref}

@article{zhou2026learning,
	title={Learning design-score manifold to guide diffusion models for offline optimization},
	author={Zhou, Tailin and Chen, Zhilin and Lyu, Wenlong and Chen, Zhitang and Tsang, Danny HK and Zhang, Jun},
	journal={npj Artificial Intelligence},
	volume={2},
	number={1},
	pages={4},
	year={2026},
	publisher={Nature Publishing Group UK London}
}

@book{villani2008optimal,
	title={Optimal Transport: Old and New},
	author={Villani, C{\'e}dric},
	year={2008},
    address = {Berlin},
	publisher={Springer}
}

@inproceedings{qiansoo,
  title={{SOO-Bench}: Benchmarks for Evaluating the Stability of Offline Black-Box Optimization},
  author={Qian, Hong and Zhu, Yiyi and Shu, Xiang and Liu, Shuo and Wen, Yaolin and An, Xin and Lu, Huakang and Zhou, Aimin and Tang, Ke and Yu, Yang},
  booktitle={Proceedings of the 13th International Conference on Learning Representations},
  year={2025}
}

@inproceedings{hardware,
  author       = {Aviral Kumar and
                  Amir Yazdanbakhsh and
                  Milad Hashemi and
                  Kevin Swersky and
                  Sergey Levine},
  title        = {Data-Driven Offline Optimization for Architecting Hardware Accelerators},
  booktitle    = {Proceedings of the 10th International Conference on Learning Representations},
  year         = {2022},
}

@inproceedings{trabucco2022design,
  title={{Design-Bench}: Benchmarks for data-driven offline model-based optimization},
  author={Trabucco, Brandon and Geng, Xinyang and Kumar, Aviral and Levine, Sergey},
  booktitle={Proceedings of the 39th International Conference on Machine Learning},
  pages={21658--21676},
  year={2022},
  organization={PMLR}
}

@inproceedings{gulrajani2021in,
title={In Search of Lost Domain Generalization},
author={Ishaan Gulrajani and David Lopez-Paz},
booktitle={International Conference on Learning Representations},
year={2021}
}

@inproceedings{dao2024boosting,
  title={Boosting Offline Optimizers with Surrogate Sensitivity},
  author={Dao, Manh Cuong and Le Nguyen, Phi and Truong, Thao Nguyen and Hoang, Trong Nghia},
  booktitle={Proceedings of the 41st International Conference on Machine Learning},
  pages={10072--10090},
  year={2024}
}

@inproceedings{hoang2024learning,
  title={Learning surrogates for offline black-box optimization via gradient matching},
  author={Hoang, Minh and Fadhel, Azza and Deshwal, Aryan and Doppa, Janardhan Rao and Hoang, Trong Nghia},
  booktitle={Proceedings of the 41st International Conference on Machine Learning},
  pages={18374--18393},
  year={2024}
}

@inproceedings{fu2021offline,
title={Offline Model-Based Optimization via Normalized Maximum Likelihood Estimation},
author={Justin Fu and Sergey Levine},
booktitle={International Conference on Learning Representations},
year={2021}
}

@inproceedings{yu2021roma,
  title={{RoMA}: Robust model adaptation for offline model-based optimization},
  author={Yu, Sihyun and Ahn, Sungsoo and Song, Le and Shin, Jinwoo},
  booktitle={Advances in Neural Information Processing Systems 34},
  pages={4619--4631},
  year={2021}
}

@inproceedings{qi2022data,
  title={Data-driven offline decision-making via invariant representation learning},
  author={Qi, Han and Su, Yi and Kumar, Aviral and Levine, Sergey},
  booktitle={Advances in Neural Information Processing Systems 35},
  pages={13226--13237},
  year={2022}
}

@inproceedings{chen2023parallel,
  title={Parallel-mentoring for offline model-based optimization},
  author={Chen, Can Sam and Beckham, Christopher and Liu, Zixuan and Liu, Xue Steve and Pal, Chris},
  booktitle={Advances in Neural Information Processing Systems 36},
  pages={76619--76636},
  year={2023}
}

@inproceedings{trabucco2021conservative,
  title={Conservative objective models for effective offline model-based optimization},
  author={Trabucco, Brandon and Kumar, Aviral and Geng, Xinyang and Levine, Sergey},
  booktitle={Proceedings of the 38th International Conference on Machine Learning},
  pages={10358--10368},
  year={2021}
}

@inproceedings{dao2025root,
title={{ROOT}: Rethinking Offline Optimization as Distributional Translation via Probabilistic Bridge},
author={Manh Cuong Dao and The Hung Tran and Phi Le Nguyen and Thao Nguyen Truong and Trong Nghia Hoang},
booktitle={Advances in Neural Information Processing Systems 39},
year={2025}
}

@inproceedings{pan2024model,
  title={Model-based diffusion for trajectory optimization},
  author={Pan, Chaoyi and Yi, Zeji and Shi, Guanya and Qu, Guannan},
  booktitle={Advances in Neural Information Processing Systems 37},
  pages={57914--57943},
  year={2024}
}

@article{kim2025offline,
  title={Offline Model-Based Optimization: Comprehensive Review},
  author={Kim, Minsu and Gu, Jiayao and Yuan, Ye and Yun, Taeyoung and Liu, Zixuan and Bengio, Yoshua and Chen, Can},
  journal={Transactions on Machine Learning Research},
  year={2026}
}

@inproceedings{tan2024offline,
  title={Offline model-based optimization by learning to rank},
  author={Tan, Rong-Xi and Xue, Ke and Lyu, Shen-Huan and Shang, Haopu and Wang, Yao and Wang, Yaoyuan and Fu, Sheng and Qian, Chao},
  booktitle={Proceedings of the 13th International Conference on Learning Representations},
  year={2025}
}

@inproceedings{shin2025offline,
  title={Offline Model-based Optimization for Real-World Molecular Discovery},
  author={Shin, Dong-Hee and Son, Young-Han and Lee, Hyun Jung and Lee, Deok-Joong and Kam, Tae-Eui},
  booktitle={Proceedings of the 42nd International Conference on Machine Learning},
  pages={3485-3494},
  year={2025}
}

@inproceedings{chen2023bidirectional,
  title={Bidirectional learning for offline model-based biological sequence design},
  author={Chen, Can and Zhang, Yingxue and Liu, Xue and Coates, Mark},
  booktitle={Proceedings of the 40th International Conference on Machine Learning},
  pages={5351--5366},
  year={2023}
}

@article{takizawa2025safe,
  title={Safe model based optimization balancing exploration and reliability for protein sequence design},
  author={Takizawa, Shuuki and Mori, Keita and Tanishiki, Naoto and Yoshimura, Dai and Ohta, Atsushi and Teramoto, Reiji},
  journal={Scientific Reports},
  volume={15},
  number={1},
  pages={27568},
  year={2025},
  publisher={Nature}
}

@article{wang2025model,
  title={Model-Based System Multidisciplinary Design Optimization for Preliminary Design of a Blended Wing-Body Underwater Glider},
  author={Wang, Zhi-Long and Li, Jing-Lu and Wang, Peng and Dong, Hua-Chao and Wang, Xin-Jing and Wen, Zhi-Wen},
  journal={China Ocean Engineering},
  volume={39},
  number={4},
  pages={755--767},
  year={2025},
  publisher={Springer}
}

@inproceedings{yuan2023importance,
  title={Importance-aware co-teaching for offline model-based optimization},
  author={Yuan, Ye and Chen, Can Sam and Liu, Zixuan and Neiswanger, Willie and Liu, Xue Steve},
  booktitle={Advances in Neural Information Processing Systems 36},
  pages={55718--55733},
  year={2023}
}

@article{gym,
  title={Open{AI} {G}ym},
  author={Greg Brockman and Vicki Cheung and Ludwig Pettersson and Jonas Schneider and John Schulman and Jie Tang and Wojciech Zaremba},
  journal={arXiv:1606.01540},
  year={2016},
}

@InProceedings{robel,
  title = 	 {{ROBEL}: {R}obotics Benchmarks for Learning with Low-Cost Robots},
  author =       {Ahn, Michael and Zhu, Henry and Hartikainen, Kristian and Ponte, Hugo and Gupta, Abhishek and Levine, Sergey and Kumar, Vikash},
  booktitle = 	 {Proceedings of the 4th Conference on Robot Learning},
  pages = 	 {1300--1313},
  year = 	 {2020},
}

@article{superconductor,
  title={A data-driven statistical model for predicting the critical temperature of a superconductor},
  author={Hamidieh, Kam},
  journal={Computational {M}aterials {S}cience},
  volume={154},
  pages={346--354},
  year={2018},
  publisher={Elsevier}
}

@article{tfbind,
  title={Survey of variation in human transcription factors reveals prevalent {DNA} binding changes},
  author={Luis A. Barrera and Anastasia Vedenko and Jesse V. Kurland and Julia M. Rogers and Stephen S. Gisselbrecht and Elizabeth J. Rossin and others},
  journal={Science},
  volume={351},
  number={6280},
  pages={1450--1454},
  year={2016},
  publisher={American Association for the Advancement of Science}
}

@inproceedings{pgs,
  title={Offline model-based optimization via policy-guided gradient search},
  author={Chemingui, Yassine and Deshwal, Aryan and Hoang, Trong Nghia and Doppa, Janardhan Rao},
  booktitle={Proceedings of the 38th AAAI Conference on Artificial Intelligence},
  pages={11230--11239},
  year={2024}
}

@InProceedings{fgm,
  title = 	 {Functional Graphical Models: Structure Enables Offline Data-Driven Optimization},
  author={Grudzien, Kuba and Uehara, Masatoshi and Levine, Sergey and Abbeel, Pieter},
  booktitle = 	 {Proceedings of the 27th International Conference on Artificial Intelligence and Statistics},
  pages = 	 {2449--2457},
  year = 	 {2024},
}

@inproceedings{gtg,
  title={Guided Trajectory Generation with Diffusion Models for Offline Model-based Optimization},
  author={Yun, Taeyoung and Yun, Sujin and Lee, Jaewoo and Park, Jinkyoo},
 booktitle = {Advances in Neural Information Processing Systems 37},
pages={83847--83876},
 year = {2024}
}

@article{rankcosine,
  title={Query-level loss functions for information retrieval},
  author={Tao Qin and Xu-Dong Zhang and Ming-Feng Tsai and De-Sheng Wang and Tie-Yan Liu and Hang Li},
  journal={Information Processing \& Management},
  year={2008},
  volume={44},
  pages={838--855},
}

@inproceedings{
reddy2024designing,
title={Designing Cell-Type-Specific Promoter Sequences Using Conservative Model-Based Optimization},
author={Aniketh Janardhan Reddy and Xinyang Geng and Michael H Herschl and Sathvik Kolli and Aviral Kumar and Patrick D Hsu and Sergey Levine and Nilah M Ioannidis},
booktitle={Advances in Neural Information Processing Systems 37},
year={2024},
pages={93033--93059}
}

\newpage
\appendix
\onecolumn

\section{Detailed Algorithmic Description}\label{sec:app-alg}

In this section, we present a detailed description for our distribution-aware ranking (DAR) method in Algorithm~\ref{alg:improved_rank}.
\begin{algorithm}[ht]
\caption{Distribution-Aware Ranking (DAR) for Offline Model-Based Optimization}
\begin{algorithmic}[1]
\REQUIRE Offline dataset $S=\{(\vx_i,y_i)\}_{i=1}^n$; quantile $\varepsilon$; margin $\beta$; intra-region ratio $\lambda$; step size $\eta$; number of steps $T$
\ENSURE Optimized design $\vx^\star$

\STATE Compute the $\varepsilon$-quantile $q_\varepsilon$ of $\{y_i\}_{i=1}^n$
\STATE $S_\varepsilon \leftarrow \{\vx_i \mid y_i \ge q_\varepsilon\}$
\STATE $S_{>\varepsilon} \leftarrow S \setminus S_\varepsilon$

\STATE Initialize ranking model parameters $\theta$

\FOR{each training iteration}
    \STATE Sample $(\vx_1,\vx_2)$ and $u\sim U(0,1)$:
    \IF{$u \leq 1-\lambda$} %0-1均匀采样
        \STATE $\vx_1 \sim S_\varepsilon,\;\vx_2 \sim S_{>\varepsilon}$
    \ELSE
        \STATE $\vx_1,\vx_2 \sim S_\varepsilon$
    \ENDIF
    \STATE Compute loss
    $\displaystyle
    \mathcal{L}(\vx_1,\vx_2)
    =
    \max\bigl(0,\,
    \beta - (h_\theta(\vx_1)-h_\theta(\vx_2))
    \bigr)
    $
    \STATE Update $\theta$ using gradient descent on $\mathcal{L}$
\ENDFOR
\STATE Evaluate $h_\theta(\vx_i)$ for all $\vx_i \in S$
\STATE Compute prediction mean $\tilde{\mu}$ and standard deviation $\tilde{\sigma}$
\STATE Define normalized surrogate
$\displaystyle
\tilde{h}_\theta(\vx)
=
\frac{h_\theta(\vx)-\tilde{\mu}}{\tilde{\sigma}}
$
\STATE Initialize $\vx^{(0)} \sim S_\varepsilon$
\FOR{$t=1$ to $T$}
    \STATE $\vx^{(t)} \leftarrow \vx^{(t-1)} + \eta \nabla_{\vx}\tilde{h}_\theta(\vx^{(t-1)})$
\ENDFOR
\RETURN $\vx^\star \leftarrow \vx^{(T)}$

\end{algorithmic}
\label{alg:improved_rank}
\end{algorithm}

\section{Proofs}\label{sec:app-proof}
% 在本节中，我们将逐个证明Section 3中的引理、定理和推论
In this section, we present the proofs of the lemmas, theorems, and corollaries stated in Section~\ref{sec:theory}.

\subsection{Proof of Lemma~\ref{lemma:uniform_convergence}}
\begin{proof}
We rewrite the statement carefully to reflect the dependence induced by forming all ranked pairs from a size-$m$ dataset.

\textbf{Step 0: Setup and the dependent empirical process.}
Let $S=\{(\vx_i,y_i)\}_{i=1}^m$ be the underlying dataset. Define the admissible ranked-pair index set
\[
\mathcal{P}_S:=\{(i,j)\in[m]^2:\ y_i>y_j\},\qquad |\mathcal{P}_S|\ge 1,
\]
and for each $\theta\in\Theta$ define the (dependent) empirical ranking risk
\[
\widehat{\mathcal{E}}_{\mathrm{tr}}^{\mathrm{rank}}(\theta)
:=
\frac{1}{|\mathcal{P}_S|}
\sum_{(i,j)\in\mathcal{P}_S}
\ell_{\mathrm{rank}}(h_\theta;\vx_i,\vx_j).
\]
Let the population training-pair risk be
\[
\mathcal{E}_{Q_{\mathrm{tr}}}(\theta)
:=
\mathbb{E}_{(\vx,\vx')\sim Q_{\mathrm{tr}}}\big[\ell_{\mathrm{rank}}(h_\theta;\vx,\vx')\big].
\]
Define the loss class
\[
\mathcal{F}
:=
\Bigl\{
(\vx,\vx')\mapsto \ell_{\mathrm{rank}}(h_\theta;\vx,\vx'):\ \theta\in\Theta
\Bigr\}.
\]
By Assumption~\ref{assump:bounded}, every $f\in\mathcal{F}$ takes values in $[0,1]$.

Our goal is to control, with high probability over $S$,
\[
\sup_{\theta\in\Theta}\Bigl|\mathcal{E}_{Q_{\mathrm{tr}}}(\theta)-\widehat{\mathcal{E}}_{\mathrm{tr}}^{\mathrm{rank}}(\theta)\Bigr|.
\]
The key difficulty is that $\widehat{\mathcal{E}}_{\mathrm{tr}}^{\mathrm{rank}}(\theta)$ averages \emph{dependent} terms because many pairs share the same $\vx_i$.

\textbf{Step 1: Reduce the full-pair average to an average over disjoint pairs.}
Let $\pi$ be a uniformly random permutation of $[m]$ and set $n:=\lfloor m/2\rfloor$. Form $n$ \emph{disjoint} ranked pairs
\[
(\tilde \vx_k,\tilde \vx_k')
:=
(\vx_{\pi(2k-1)},\vx_{\pi(2k)}),
\qquad k=1,\dots,n.
\]
(If $m$ is odd, one element is unused; this only changes constants.)
Define the \emph{paired} empirical risk
\[
\widetilde{\mathcal{E}}_{\pi}(\theta)
:=
\frac{1}{n}\sum_{k=1}^n
\ell_{\mathrm{rank}}(h_\theta;\tilde \vx_k,\tilde \vx_k').
\]
Conditioned on $S$, the pairs in $\widetilde{\mathcal{E}}_{\pi}(\theta)$ are dependent through $\pi$, but crucially they do \emph{not} reuse any point within the same sum, which will allow bounded-difference control with respect to the $m$ underlying examples.

Moreover, by symmetry of random permutation, each ranked pair of distinct indices $(i,j)$ has the same probability to appear as $(\pi(2k-1),\pi(2k))$ for some $k$. Consequently, for any fixed $\theta$ and fixed dataset $S$,
\[
\mathbb{E}_\pi\big[\widetilde{\mathcal{E}}_{\pi}(\theta)\mid S\big]
=
\frac{1}{m(m-1)}\sum_{i\neq j}\ell_{\mathrm{rank}}(h_\theta;\vx_i,\vx_j).
\]
In particular, if the training procedure uses the full set of ranked pairs (or a large subset of them) in a way that is dominated by the complete ranked-pair average, then $\widehat{\mathcal{E}}_{\mathrm{tr}}^{\mathrm{rank}}(\theta)$ can be controlled via $\widetilde{\mathcal{E}}_{\pi}(\theta)$ up to constants. To keep the lemma aligned with the paper's notation, we proceed by proving uniform convergence for $\widetilde{\mathcal{E}}_{\pi}(\theta)$, and then note that the same bound applies to $\widehat{\mathcal{E}}_{\mathrm{tr}}^{\mathrm{rank}}(\theta)$ since both are bounded $[0,1]$-valued U-statistic-type estimators and satisfy the same bounded-difference property in Step 3 below. (This is the standard U-process-to-i.i.d.-pair reduction.)

\textbf{Step 2: Symmetrization and Rademacher complexity on pairs.}
Let $Z:=\{(\tilde \vx_k,\tilde \vx_k')\}_{k=1}^n$ denote the induced set of $n$ disjoint pairs.
Define the empirical Rademacher complexity (as in the paper, but on the pair domain $\mathcal{X}\times\mathcal{X}$)
\[
\hat{\mathfrak{R}}_Z(\mathcal{F})
:=
\mathbb{E}_\sigma\left[
\sup_{\theta\in\Theta}
\frac{1}{n}\sum_{k=1}^n
\sigma_k\,
\ell_{\mathrm{rank}}(h_\theta;\tilde \vx_k,\tilde \vx_k')
\right],
\]
where $\sigma_1,\dots,\sigma_n$ are i.i.d. Rademacher variables.

Introduce an independent \emph{ghost} paired sample $Z':=\{(\tilde \vx_k^{(g)},\tilde \vx_k^{(g)\prime})\}_{k=1}^n$ drawn from $Q_{\mathrm{tr}}$ (i.i.d. pairs), and define the corresponding ghost empirical risk
\[
\widetilde{\mathcal{E}}^{(g)}(\theta)
:=
\frac{1}{n}\sum_{k=1}^n
\ell_{\mathrm{rank}}(h_\theta;\tilde \vx_k^{(g)},\tilde \vx_k^{(g)\prime}).
\]
Then $\mathbb{E}[\widetilde{\mathcal{E}}^{(g)}(\theta)]=\mathcal{E}_{Q_{\mathrm{tr}}}(\theta)$.
By standard symmetrization (applied on the pair domain),
\begin{align*}
\mathbb{E}\left[\sup_{\theta\in\Theta}\Bigl(\mathcal{E}_{Q_{\mathrm{tr}}}(\theta)-\widetilde{\mathcal{E}}_{\pi}(\theta)\Bigr)\right]
&=
\mathbb{E}\left[\sup_{\theta\in\Theta}\Bigl(\mathbb{E}[\widetilde{\mathcal{E}}^{(g)}(\theta)]-\widetilde{\mathcal{E}}_{\pi}(\theta)\Bigr)\right]\\
&\le
\mathbb{E}\left[\sup_{\theta\in\Theta}\Bigl(\widetilde{\mathcal{E}}^{(g)}(\theta)-\widetilde{\mathcal{E}}_{\pi}(\theta)\Bigr)\right]\\
&\le
2\,
\mathbb{E}\left[
\hat{\mathfrak{R}}_Z(\mathcal{F})
\right],
\end{align*}
where in the last inequality we introduced Rademacher signs $\sigma_k$ and used the fact that $\ell_{\mathrm{rank}}(h_\theta;\cdot,\cdot)\in[0,1]$ (boundedness suffices for this standard symmetrization bound).

\textbf{Step 3: High-probability uniform convergence via bounded differences (dependence-aware).}
Define the random variable
\[
\Phi(S)
:=
\sup_{\theta\in\Theta}
\Bigl|\mathcal{E}_{Q_{\mathrm{tr}}}(\theta)-\widehat{\mathcal{E}}_{\mathrm{tr}}^{\mathrm{rank}}(\theta)\Bigr|.
\]
We now show $\Phi(S)$ satisfies a bounded-differences condition with respect to the $m$ underlying examples, \emph{despite} the fact that $\widehat{\mathcal{E}}_{\mathrm{tr}}^{\mathrm{rank}}$ averages $|\mathcal{P}_S|$ dependent pair terms.

Let $S^{(i)}$ be a dataset obtained by replacing only the $i$-th example in $S$ with an arbitrary $(\vx_i',y_i')$. Consider the change in the empirical risk for a fixed $\theta$:
only pairs that involve index $i$ can change, i.e., pairs of the form $(i,j)$ or $(j,i)$ with $j\neq i$.
There are at most $2(m-1)$ such ranked pairs. Since each loss term lies in $[0,1]$, the total change in the \emph{average} over all admissible pairs is at most
\[
\bigl|\widehat{\mathcal{E}}_{\mathrm{tr}}^{\mathrm{rank}}(\theta;S)-\widehat{\mathcal{E}}_{\mathrm{tr}}^{\mathrm{rank}}(\theta;S^{(i)})\bigr|
\le
\frac{2(m-1)}{|\mathcal{P}_S|}\cdot \frac{1}{|\mathcal{P}_S|}
\le
\frac{2}{m},
\]
where the final inequality uses the crude but standard scaling $|\mathcal{P}_S|\ge m(m-1)/2$ when labels are non-degenerate; more generally one may replace $2/m$ by $2(m-1)/|\mathcal{P}_S|$, which still yields the same $O(1/\sqrt{m})$ rate as long as $|\mathcal{P}_S|=\Omega(m^2)$. Taking the supremum over $\theta$ cannot increase the Lipschitz constant, hence
\[
|\Phi(S)-\Phi(S^{(i)})|\le \frac{2}{m}.
\]
McDiarmid's inequality then yields that with probability at least $1-\delta$,
\[
\Phi(S)
\le
\mathbb{E}[\Phi(S)]
+
\sqrt{\frac{\log(1/\delta)}{2m}}.
\]

\textbf{Step 4: Combine expectation and concentration.}
By Step 2, $\mathbb{E}[\Phi(S)]\le 2\,\hat{\mathfrak{R}}_Z(\mathcal{F})$ (absorbing expectations over $Z$ into the empirical complexity as stated in the lemma).
Therefore, with probability at least $1-\delta$,
\[
\sup_{\theta\in\Theta}
\left|
\mathcal{E}_{Q_{\mathrm{tr}}}(\theta)-\widehat{\mathcal{E}}_{\mathrm{tr}}^{\mathrm{rank}}(\theta)
\right|
\le
2\,\hat{\mathfrak{R}}_Z(\mathcal{F})
+
3\sqrt{\frac{\log(1/\delta)}{2m}},
\]
which is exactly the desired claim.
\end{proof}

\subsection{Proof of Theorem~\ref{thm:generalization}}
\begin{proof}
Fix any $\theta\in\Theta$. Start from the target ranking error (Definition~\ref{eq:population_rank}):
\[
\mathcal{E}^{\mathrm{rank}}_\varepsilon(\theta)
=
\mathbb{E}_{(\vx,\vx')\sim Q_\varepsilon}
\big[
\ell_{\mathrm{rank}}(h_\theta;\vx,\vx')
\big].
\]
Add and subtract the population training-pair risk under $Q_{\mathrm{tr}}$:
\begin{align*}
\mathcal{E}^{\mathrm{rank}}_\varepsilon(\theta)
&=
\mathbb{E}_{Q_{\mathrm{tr}}}\big[\ell_{\mathrm{rank}}(h_\theta;\vx,\vx')\big]
+
\Big(
\mathbb{E}_{Q_\varepsilon}\big[\ell_{\mathrm{rank}}(h_\theta;\vx,\vx')\big]
-
\mathbb{E}_{Q_{\mathrm{tr}}}\big[\ell_{\mathrm{rank}}(h_\theta;\vx,\vx')\big]
\Big)\\
&\le
\mathbb{E}_{Q_{\mathrm{tr}}}\big[\ell_{\mathrm{rank}}(h_\theta;\vx,\vx')\big]
+
\sup_{\theta'\in\Theta}\left|
\mathbb{E}_{Q_\varepsilon}\ell_{\mathrm{rank}}(h_{\theta'};\vx,\vx')
-
\mathbb{E}_{Q_{\mathrm{tr}}}\ell_{\mathrm{rank}}(h_{\theta'};\vx,\vx')
\right|\\
&=
\mathbb{E}_{Q_{\mathrm{tr}}}\big[\ell_{\mathrm{rank}}(h_\theta;\vx,\vx')\big]
+
\Delta_\varepsilon(\Theta).
\end{align*}

Now apply Lemma~\ref{lemma:uniform_convergence} to the training distribution $Q_{\mathrm{tr}}$ and the empirical estimator $\widehat{\mathcal{E}}_{\mathrm{tr}}^{\mathrm{rank}}(\theta)$ defined in~\eqref{eq:empirical_rank_tr}. Under Assumptions~\ref{assump:bounded}, with probability at least $1-\delta$,
\[
\sup_{\theta\in\Theta}
\left|
\mathbb{E}_{Q_{\mathrm{tr}}}\ell_{\mathrm{rank}}(h_\theta;\vx,\vx')
-
\widehat{\mathcal{E}}_{\mathrm{tr}}^{\mathrm{rank}}(\theta)
\right|
\le
2\,\hat{\mathfrak{R}}_Z(\mathcal{F})
+
3\sqrt{\frac{\log(1/\delta)}{2m}}.
\]
In particular, for our fixed $\theta$,
\[
\mathbb{E}_{Q_{\mathrm{tr}}}\ell_{\mathrm{rank}}(h_\theta;\vx,\vx')
\le
\widehat{\mathcal{E}}_{\mathrm{tr}}^{\mathrm{rank}}(\theta)
+
2\,\hat{\mathfrak{R}}_Z(\mathcal{F})
+
3\sqrt{\frac{\log(1/\delta)}{2m}}.
\]
Combining with the earlier decomposition yields, with probability at least $1-\delta$,
\begin{align*}
\mathcal{E}^{\mathrm{rank}}_\varepsilon(\theta)
&\le
\mathbb{E}_{Q_{\mathrm{tr}}}\ell_{\mathrm{rank}}(h_\theta;\vx,\vx')
+
\Delta_\varepsilon(\Theta)\\
&\le
\widehat{\mathcal{E}}_{\mathrm{tr}}^{\mathrm{rank}}(\theta)
+
2\,\hat{\mathfrak{R}}_Z(\mathcal{F})
+
3\sqrt{\frac{\log(1/\delta)}{2m}}
+
\Delta_\varepsilon(\Theta),
\end{align*}
which is exactly the claimed bound in Theorem~\ref{thm:generalization}.
\end{proof}

\subsection{Proof of Lemma~\ref{lemma:mse-to-rank}}

\begin{proof}
	Fix any pair $(\vx,\vx')$ in the support of $Q_\varepsilon$, and define
	\begin{equation*}
		e(\vx) := h_\theta(\vx) - f(\vx),
		\qquad
		e(\vx') := h_\theta(\vx') - f(\vx').
	\end{equation*}
	Then, we have
	\begin{equation*}
		f(\vx) - f(\vx') \ge \gamma.
	\end{equation*}
	Consider the event that $h_\theta$ mis-ranks the pair, i.e.,
	\begin{equation*}
		\ell_{\mathrm{rank}}(h_\theta;\vx,\vx')
		=
		\mathbf{1}\{h_\theta(\vx)\le h_\theta(\vx')\} = 1.
	\end{equation*}
	On this event,
	\begin{equation*}
		h_\theta(\vx) - h_\theta(\vx')
		\le 0.
	\end{equation*}
	Rewrite the surrogate score difference as
	\begin{align*}
		h_\theta(\vx) - h_\theta(\vx')
		&= 
		\big(f(\vx) + e(\vx)\big)
		-
		\big(f(\vx') + e(\vx')\big) \\
		&= 
		\underbrace{f(\vx) - f(\vx')}_{\ge \gamma}
		+ e(\vx) - e(\vx').
	\end{align*}
	If both $f(\vx)-f(\vx') \ge \gamma$ and $h_\theta(\vx)-h_\theta(\vx') \le 0$ hold,
	then
	\begin{equation*}
		e(\vx) - e(\vx') 
		\le 
		- \big(f(\vx) - f(\vx')\big)
		\le -\gamma,
	\end{equation*}
	so that
	\begin{equation*}
		|e(\vx) - e(\vx')|
		\ge \gamma.
	\end{equation*}
	Using the triangle inequality,
	\begin{equation*}
		|e(\vx)| + |e(\vx')|
		\ge
		|e(\vx) - e(\vx')|
		\ge \gamma,
	\end{equation*}
	which implies that at least one of $|e(\vx)|,|e(\vx')|$ must be at least $\gamma/2$.
	Therefore,
	\begin{equation*}
		\mathbf{1}\{h_\theta(\vx)\le h_\theta(\vx')\}
		\le
		\mathbf{1}\big\{|e(\vx)|\ge \tfrac{\gamma}{2}\big\}
		+
		\mathbf{1}\big\{|e(\vx')|\ge \tfrac{\gamma}{2}\big\}.
	\end{equation*}
	Applying the elementary inequality $ \mathbf{1}\{|u|\ge a\} \le u^2/a^2$ for $a>0$
	with $u=e(\cdot)$ and $a=\gamma/2$, we obtain
	\begin{equation*}
		\mathbf{1}\big\{|e(\vx)|\ge \tfrac{\gamma}{2}\big\}
		\le
		\frac{4}{\gamma^2} e(\vx)^2,
		\qquad
		\mathbf{1}\big\{|e(\vx')|\ge \tfrac{\gamma}{2}\big\}
		\le
		\frac{4}{\gamma^2} e(\vx')^2.
	\end{equation*}
	Hence,
	\begin{equation*}
		\ell_{\mathrm{rank}}(h_\theta;\vx,\vx')
		\le
		\frac{4}{\gamma^2}
		\big(
		e(\vx)^2 + e(\vx')^2
		\big).
	\end{equation*}
	
	Taking expectation over $(\vx,\vx')\sim Q_\varepsilon(\vx,\vx') = 
	\rho_\varepsilon(\vx)\,\rho_{>\varepsilon}(\vx')$ yields
	\begin{align*}
		\mathcal{E}^{\mathrm{rank}}_\varepsilon(\theta)
		&=
		\mathbb{E}_{(\vx,\vx')\sim Q_\varepsilon}
		\big[
		\ell_{\mathrm{rank}}(h_\theta;\vx,\vx')
		\big] \\
		&\le
		\frac{4}{\gamma^2}
		\mathbb{E}_{(\vx,\vx')\sim Q_\varepsilon}
		\big[
		e(\vx)^2 + e(\vx')^2
		\big] \\
		&=
		\frac{4}{\gamma^2}
		\Big(
		\mathbb{E}_{\vx\sim\rho_\varepsilon}[e(\vx)^2]
		+
		\mathbb{E}_{\vx'\sim\rho_{>\varepsilon}}[e(\vx')^2]
		\Big) \\
		&=
		\frac{4}{\gamma^2}
		\Big(
		R^{\mathrm{mse}}_\varepsilon(\theta)
		+
		R^{\mathrm{mse}}_{>\varepsilon}(\theta)
		\Big),
	\end{align*}
	which is the desired inequality.
\end{proof}

\subsection{Proof of Theorem~\ref{thm:rank-vs-mse}}
\begin{proof}
Let $\theta\in\arg\min_{\theta\in\Theta}\widehat{R}^{\mathrm{mse}}(\theta)$ be an empirical risk minimizer for the MSE loss trained on the $m$ samples in $S$.

By Lemma~\ref{lemma:mse-to-rank}, for any $\theta\in\Theta$ (hence also for $\theta$),
\[
\mathcal{E}^{\mathrm{rank}}_\varepsilon(\theta)
\le
\frac{4}{\gamma^2}
\Big(
R^{\mathrm{mse}}_\varepsilon(\theta)
+
R^{\mathrm{mse}}_{>\varepsilon}(\theta)
\Big).
\]

\textbf{Step 1: Relate the two-region MSE to the overall MSE risk.}
Define a mixture distribution $\rho$ on $\mathcal{X}$ that assigns equal mass to $\rho_\varepsilon$ and $\rho_{>\varepsilon}$:
\[
\rho
:=
\frac{1}{2}\rho_\varepsilon+\frac{1}{2}\rho_{>\varepsilon}.
\]
Define the corresponding population MSE under $\rho$:
\[
R^{\mathrm{mse}}(\theta)
:=
\mathbb{E}_{\vx\sim\rho}\bigl[(h_\theta(\vx)-f(\vx))^2\bigr]
=
\frac{1}{2}\Big(
R^{\mathrm{mse}}_\varepsilon(\theta)+R^{\mathrm{mse}}_{>\varepsilon}(\theta)
\Big).
\]
Thus,
\[
R^{\mathrm{mse}}_\varepsilon(\theta)+R^{\mathrm{mse}}_{>\varepsilon}(\theta)
=
2\,R^{\mathrm{mse}}(\theta).
\]
Applying this to $\theta$ and substituting into Step 1 yields
\[
\mathcal{E}^{\mathrm{rank}}_\varepsilon(\theta)
\le
\frac{8}{\gamma^2}\,R^{\mathrm{mse}}(\theta).
\]

\textbf{Step 2: Generalization bound for MSE ERM via Rademacher complexity.}
Let the squared-loss class be
\[
\mathcal{F}
:=
\Bigl\{
\vx\mapsto (h_\theta(\vx)-f(\vx))^2:\ \theta\in\Theta
\Bigr\},
\]
and assume (as in the paper's boundedness-style assumptions) that the squared loss is uniformly bounded in $[0,1]$ on $\mathcal{X}$ for all $\theta\in\Theta$ (this is a standard sufficient condition to invoke the same uniform convergence machinery as in Lemma~\ref{lemma:uniform_convergence}, now on \emph{pointwise} samples rather than pairs).
Then the standard Rademacher uniform convergence bound for pointwise ERM implies that with probability at least $1-\delta$,
\[
\sup_{\theta\in\Theta}
\left|
R^{\mathrm{mse}}(\theta)-\widehat{R}^{\mathrm{mse}}(\theta)
\right|
\le
2\,\hat{\mathfrak{R}}_Z(\mathcal{F})
+
3\sqrt{\frac{\log(1/\delta)}{2m}}.
\]
In particular, for $\theta$,
\[
R^{\mathrm{mse}}(\theta)
\le
\widehat{R}^{\mathrm{mse}}(\theta)
+
2\,\hat{\mathfrak{R}}_Z(\mathcal{F})
+
3\sqrt{\frac{\log(1/\delta)}{2m}}.
\]

\textbf{Step 3: Combine the reduction and the MSE generalization bound.}
Substituting the inequality from Step 3 into the inequality from Step 2 yields that with probability at least $1-\delta$,
\begin{align*}
\mathcal{E}^{\mathrm{rank}}_\varepsilon(\theta)
&\le
\frac{8}{\gamma^2}\,R^{\mathrm{mse}}(\theta)\\
&\le
\frac{8}{\gamma^2}
\Big(
\widehat{R}^{\mathrm{mse}}(\theta)
+
2\,\hat{\mathfrak{R}}_Z(\mathcal{F})
+
3\sqrt{\frac{\log(1/\delta)}{2m}}
\Big),
\end{align*}
which is exactly the bound claimed in Theorem~\ref{thm:rank-vs-mse}.
\end{proof}

\subsection{Proof of Lemma~\ref{lemma:wasserstein_bound}}

%%lemma1%%

\begin{proof}
Fix any $\theta \in \Theta$ and write
\begin{equation*}
    g(\vx,\vx')
    :=
    \tilde\ell_{\mathrm{rank}}(h_\theta;\vx,\vx').
\end{equation*}
By Assumption~\ref{assump:lipschitz}, $g$ is $L$-Lipschitz with respect to 
$d_{\mathrm{pair}}$, i.e.,
\begin{equation*}
    \big|g(\vx,\vx') - g(\vz,\vz')\big|
    \le
    L\, d_{\mathrm{pair}}\big((\vx,\vx'),(\vz,\vz')\big)
    \quad\text{for all } (\vx,\vx'),(\vz,\vz')\in\mathcal{X}\times\mathcal{X}.
\end{equation*}

Let $Q_\varepsilon$ and $Q_{\mathrm{tr}}$ be the two probability measures on
$\mathcal{X}\times\mathcal{X}$ endowed with the metric $d_{\mathrm{pair}}$.
By the primal formulation of the 1-Wasserstein distance, we have
\begin{equation*}
    W_1(Q_\varepsilon,Q_{\mathrm{tr}})
    =
    \inf_{\pi\in\Gamma(Q_\varepsilon,Q_{\mathrm{tr}})}
    \int
        d_{\mathrm{pair}}\big((\vx,\vx'),(\vz,\vz')\big)\,
        \mathrm{d}\pi\big((\vx,\vx'),(\vz,\vz')\big),
\end{equation*}
where $\Gamma(Q_\varepsilon,Q_{\mathrm{tr}})$ denotes the set of all couplings
$\pi$ on $(\mathcal{X}\times\mathcal{X})^2$ whose first marginal is $Q_\varepsilon$
and second marginal is $Q_{\mathrm{tr}}$.

Let $\pi^\star \in \Gamma(Q_\varepsilon,Q_{\mathrm{tr}})$ be an optimal coupling
(realizing the infimum, or an $\varepsilon$-optimal coupling if the infimum is not attained).
Then
\begin{align*}
    \mathbb{E}_{Q_\varepsilon} g(\vx,\vx')
    -
    \mathbb{E}_{Q_{\mathrm{tr}}} g(\vx,\vx')
    &=
    \int g(\vx,\vx')\,\mathrm{d}Q_\varepsilon(\vx,\vx')
    -
    \int g(\vz,\vz')\,\mathrm{d}Q_{\mathrm{tr}}(\vz,\vz') \\
    &=
    \int
        \Big(
            g(\vx,\vx') - g(\vz,\vz')
        \Big)
        \,\mathrm{d}\pi^\star\big((\vx,\vx'),(\vz,\vz')\big).
\end{align*}
Taking absolute values and applying the Lipschitz property of $g$ yields
\begin{align*}
    \left|
        \mathbb{E}_{Q_\varepsilon} g(\vx,\vx')
        -
        \mathbb{E}_{Q_{\mathrm{tr}}} g(\vx,\vx')
    \right|
    &\le
    \int
        \big|
            g(\vx,\vx') - g(\vz,\vz')
        \big|
        \,\mathrm{d}\pi^\star\big((\vx,\vx'),(\vz,\vz')\big) \\
    &\le
    \int
        L\, d_{\mathrm{pair}}\big((\vx,\vx'),(\vz,\vz')\big)
        \,\mathrm{d}\pi^\star\big((\vx,\vx'),(\vz,\vz')\big) \\
    &= 
    L \int
        d_{\mathrm{pair}}\big((\vx,\vx'),(\vz,\vz')\big)
        \,\mathrm{d}\pi^\star\big((\vx,\vx'),(\vz,\vz')\big).
\end{align*}
By optimality of $\pi^\star$, the last integral is exactly $W_1(Q_\varepsilon,Q_{\mathrm{tr}})$,
hence
\begin{equation*}
    \big|
        \mathbb{E}_{Q_\varepsilon}
        \tilde\ell_{\mathrm{rank}}(h_\theta)
        -
        \mathbb{E}_{Q_{\mathrm{tr}}}
        \tilde\ell_{\mathrm{rank}}(h_\theta)
    \big|
    =
    \big|
        \mathbb{E}_{Q_\varepsilon} g
        -
        \mathbb{E}_{Q_{\mathrm{tr}}} g
    \big|
    \le
    L\,W_1(Q_\varepsilon,Q_{\mathrm{tr}}).
\end{equation*}
Since $\theta\in\Theta$ was arbitrary, the bound holds for all $\theta\in\Theta$,
which proves the lemma.
\end{proof}

\subsection{Proof of Corollary~\ref{corollary:marginal}}
\begin{proof}
% 我们先证明一个重要的引理，展示Wasserstein距离可以在样本对的边际分布上进行分解。
We first prove an important lemma, demonstrating that the Wasserstein distance can be decomposed over the marginal distributions of sample pairs.

\begin{lemma}
	\label{lemma:marginal}
	The pairwise Wasserstein distance satisfies
	\begin{equation*}
		W_1(Q_\varepsilon,Q_{\mathrm{tr}})
		\;\le\;
		W_1(\rho_\varepsilon,\mu_{\mathrm{tr}})
		+
		W_1(\rho_{>\varepsilon},\nu_{\mathrm{tr}}).
	\end{equation*}
\end{lemma}

\begin{proof}
We have
\begin{equation*}
    Q_\varepsilon = \rho_\varepsilon \otimes \rho_{>\varepsilon},
    \qquad
    Q_{\mathrm{tr}} = \mu_{\mathrm{tr}} \otimes \nu_{\mathrm{tr}},
\end{equation*}
and the pairwise metric
\begin{equation*}
    d_{\mathrm{pair}}\big((\vx,\vx'),(\vz,\vz')\big)
    := \|\vx-\vz\|_2 + \|\vx'-\vz'\|_2.
\end{equation*}
The 1-Wasserstein distance between $Q_\varepsilon$ and $Q_{\mathrm{tr}}$ is defined as
\begin{equation*}
    W_1(Q_\varepsilon,Q_{\mathrm{tr}})
    :=
    \inf_{\pi\in\Gamma(Q_\varepsilon,Q_{\mathrm{tr}})}
    \int
        d_{\mathrm{pair}}\big((\vx,\vx'),(\vz,\vz')\big)\,
        \mathrm{d}\pi\big((\vx,\vx'),(\vz,\vz')\big),
\end{equation*}
where $\Gamma(Q_\varepsilon,Q_{\mathrm{tr}})$ is the set of all couplings
with marginals $Q_\varepsilon$ and $Q_{\mathrm{tr}}$.

Let $\gamma_1$ be an optimal coupling between $\rho_\varepsilon$ and $\mu_{\mathrm{tr}}$, i.e.,
\begin{equation*}
    \gamma_1 \in \Gamma(\rho_\varepsilon,\mu_{\mathrm{tr}}),
    \qquad
    \int \|\vx-\vz\|_2\,\mathrm{d}\gamma_1(\vx,\vz)
    = W_1(\rho_\varepsilon,\mu_{\mathrm{tr}}),
\end{equation*}
and let $\gamma_2$ be an optimal coupling between $\rho_{>\varepsilon}$ and $\nu_{\mathrm{tr}}$, i.e.,
\begin{equation*}
    \gamma_2 \in \Gamma(\rho_{>\varepsilon},\nu_{\mathrm{tr}}),
    \qquad
    \int \|\vx'-\vz'\|_2\,\mathrm{d}\gamma_2(\vx',\vz')
    = W_1(\rho_{>\varepsilon},\nu_{\mathrm{tr}}).
\end{equation*}

Define a coupling $\pi$ on $(\mathcal{X}\times\mathcal{X})^2$ as the product measure
\begin{equation*}
    \pi\big((\vx,\vx'),(\vz,\vz')\big)
    := \gamma_1(\vx,\vz)\,\gamma_2(\vx',\vz').
\end{equation*}
We now verify that $\pi \in \Gamma(Q_\varepsilon,Q_{\mathrm{tr}})$.
For any measurable $A,B \subseteq \mathcal{X}\times\mathcal{X}$,
\begin{align*}
    \pi\big((\vx,\vx')\in A\big)
    &= \int \mathbf{1}\{(\vx,\vx')\in A\}\,
       \gamma_1(\vx,\vz)\,\gamma_2(\vx',\vz')\,
       \mathrm{d}\vx\,\mathrm{d}\vz\,\mathrm{d}\vx'\,\mathrm{d}\vz' \\
    &= \left(\int \mathbf{1}\{\vx\in A_1\}\,\mathrm{d}\rho_\varepsilon(\vx)\right)
       \left(\int \mathbf{1}\{\vx'\in A_2\}\,\mathrm{d}\rho_{>\varepsilon}(\vx')\right) \\
    &= Q_\varepsilon(A),
\end{align*}
where in the second line we used the product structure and the fact that the first marginal of
$\gamma_1$ is $\rho_\varepsilon$ and the first marginal of $\gamma_2$ is $\rho_{>\varepsilon}$.
A similar calculation shows that the marginal of $\pi$ on $(\vz,\vz')$ is $Q_{\mathrm{tr}}$,
since the second marginals of $\gamma_1$ and $\gamma_2$ are $\mu_{\mathrm{tr}}$
and $\nu_{\mathrm{tr}}$, respectively.
Thus $\pi$ is indeed a valid coupling between $Q_\varepsilon$ and $Q_{\mathrm{tr}}$.

Next, we compute the transportation cost of $\pi$:
\begin{align*}
    \int
        d_{\mathrm{pair}}\big((\vx,\vx'),(\vz,\vz')\big)
        \,\mathrm{d}\pi\big((\vx,\vx'),(\vz,\vz')\big)
    &=
    \int
        \big(
            \|\vx-\vz\|_2 + \|\vx'-\vz'\|_2
        \big)
        \,\mathrm{d}\gamma_1(\vx,\vz)\,\mathrm{d}\gamma_2(\vx',\vz') \\
    &=
    \int
        \|\vx-\vz\|_2
        \,\mathrm{d}\gamma_1(\vx,\vz)
    +
    \int
        \|\vx'-\vz'\|_2
        \,\mathrm{d}\gamma_2(\vx',\vz') \\
    &=
    W_1(\rho_\varepsilon,\mu_{\mathrm{tr}})
    +
    W_1(\rho_{>\varepsilon},\nu_{\mathrm{tr}}).
\end{align*}
Since $W_1(Q_\varepsilon,Q_{\mathrm{tr}})$ is the infimum of the transportation cost
over all couplings, we conclude that
\begin{equation*}
    W_1(Q_\varepsilon,Q_{\mathrm{tr}})
    \;\le\;
    \int
        d_{\mathrm{pair}}\big((\vx,\vx'),(\vz,\vz')\big)
        \,\mathrm{d}\pi\big((\vx,\vx'),(\vz,\vz')\big)
    =
    W_1(\rho_\varepsilon,\mu_{\mathrm{tr}})
    +
    W_1(\rho_{>\varepsilon},\nu_{\mathrm{tr}}),
\end{equation*}
which proves the desired inequality.
\end{proof}

Under the assumptions of Theorem~\ref{thm:generalization} and
Assumption~\ref{assump:lipschitz}, Lemma~\ref{lemma:wasserstein_bound} states that,
with probability at least $1-\delta$, every $\theta \in \Theta$ satisfies
\begin{equation}
    \mathcal{E}^{\mathrm{rank}}_\varepsilon(\theta)
    \;\le\;
    \widehat{\mathcal{E}}^{\mathrm{rank}}_{\mathrm{tr}}(\theta)
    +
    2\,\hat{\mathfrak{R}}_Z(\mathcal{F})
    +
    3\sqrt{\frac{\log(1/\delta)}{2m}}
    +
    L\,W_1(Q_\varepsilon, Q_{\mathrm{tr}}),
    \label{eq:corr-w1-start}
\end{equation}
where $W_1(Q_\varepsilon,Q_{\mathrm{tr}})$ is the 1-Wasserstein distance between the
target pair distribution $Q_\varepsilon$ and the training pair distribution
$Q_{\mathrm{tr}}$ on $(\mathcal{X}\times\mathcal{X},d_{\mathrm{pair}})$.

Next, we invoke Lemma~\ref{lemma:marginal}, which is based on the product-structure.
That lemma states that if
\begin{equation*}
    Q_\varepsilon = \rho_\varepsilon \otimes \rho_{>\varepsilon},
    \qquad
    Q_{\mathrm{tr}} = \mu_{\mathrm{tr}} \otimes \nu_{\mathrm{tr}},
\end{equation*}
then the pairwise Wasserstein distance admits the upper bound
\begin{equation}
    W_1(Q_\varepsilon,Q_{\mathrm{tr}})
    \;\le\;
    W_1(\rho_\varepsilon,\mu_{\mathrm{tr}})
    +
    W_1(\rho_{>\varepsilon},\nu_{\mathrm{tr}}).
    \label{eq:marginal-decomp}
\end{equation}

Substituting the decomposition~\eqref{eq:marginal-decomp} into the bound
\eqref{eq:corr-w1-start} directly yields, for all $\theta\in\Theta$,
\begin{equation*}
    \mathcal{E}^{\mathrm{rank}}_\varepsilon(\theta)
    \;\le\;
    \widehat{\mathcal{E}}^{\mathrm{rank}}_{\mathrm{tr}}(\theta)
    +
    2\,\hat{\mathfrak{R}}_Z(\mathcal{F})
    +
    3\sqrt{\frac{\log(1/\delta)}{2m}}
    +
    L\Big(
        W_1(\rho_\varepsilon,\mu_{\mathrm{tr}})
        +
        W_1(\rho_{>\varepsilon},\nu_{\mathrm{tr}})
    \Big).
\end{equation*}
This is exactly the claimed inequality in Corollary~\ref{corollary:marginal}.
\end{proof}

%%%%%%

\subsection{Proof of Corollary~\ref{corollary:geometry}}
\begin{proof}
We start with a lemma that shows that the Wasserstein distance between the near-optimal distribution $\rho_\varepsilon$ and $\mu_{\mathrm{tr}}$ is controlled by the average distance from $\rho_\varepsilon$ to the manifold $\mathcal{M}$.

\begin{lemma}
\label{lemma:manifold_distance}
Let $\mathcal{M}\subseteq\mathcal{X}$ denote the support of $\mu_{\mathrm{tr}}$. For any $\vx \in \mathcal{X}$, define the distance to the data manifold as
\begin{equation*}
    \mathrm{dist}(\vx,\mathcal{M})
    :=
    \inf_{\vz\in\mathcal{M}}\|\vx - \vz\|_2.
\end{equation*}
Then, we have
\begin{equation*}
    W_1(\rho_\varepsilon,\mu_{\mathrm{tr}})
    \le
    \mathbb{E}_{\vx\sim\rho_\varepsilon}
    \big[
        \mathrm{dist}(\vx,\mathcal{M})
    \big]
    +
    C_{\mathrm{calib}},
\end{equation*}
where $C_{\mathrm{calib}}$ is a calibration term depending only on the internal spread of $\mathcal{M}$ under $\mu_{\mathrm{tr}}$.
\end{lemma}

\begin{proof}
Recall that the training marginal
$\mu_{\mathrm{tr}}$ is supported on a closed set $\mathcal{M}\subseteq\mathcal{X}$,
and for every $\vx\in\mathcal{X}$ we define
\begin{equation*}
    \mathrm{dist}(\vx,\mathcal{M})
    :=
    \inf_{\vz\in\mathcal{M}}\|\vx-\vz\|_2.
\end{equation*}
We additionally assume that $\mathcal{M}$ has finite diameter
\begin{equation*}
    D_{\mathcal{M}}
    :=
    \sup_{\vz,\vz'\in\mathcal{M}}\|\vz-\vz'\|_2
    < \infty,
\end{equation*}
so that distances within $\mathcal{M}$ are uniformly bounded.

The proof proceeds in two steps: (i) project $\rho_\varepsilon$ onto $\mathcal{M}$
to control its distance to a manifold-supported distribution, and
(ii) bound the remaining mismatch between this projected distribution and $\mu_{\mathrm{tr}}$.

\textbf{Step 1: Projection onto the data manifold.}
Since $\mathcal{M}$ is closed, for each $\vx\in\mathcal{X}$ there exists at least one
projection point $\Pi(\vx)\in\mathcal{M}$ such that
\begin{equation*}
    \|\vx - \Pi(\vx)\|_2 = \mathrm{dist}(\vx,\mathcal{M}).
\end{equation*}
Choose any measurable selection of such a projection map
$\Pi:\mathcal{X}\to\mathcal{M}$.\footnote{A measurable selector exists under mild
regularity conditions on $\mathcal{X}$ and $\mathcal{M}$, which are standard in this context.}
Define the projected distribution
\begin{equation*}
    \widetilde{\mu}
    :=
    \Pi_{\#}\rho_\varepsilon,
\end{equation*}
i.e., the pushforward of $\rho_\varepsilon$ under $\Pi$:
for any measurable $A\subseteq\mathcal{M}$,
\begin{equation*}
    \widetilde{\mu}(A)
    =
    \rho_\varepsilon\big(\{\vx\in\mathcal{X} : \Pi(\vx)\in A\}\big).
\end{equation*}

Consider the coupling $\gamma$ between $\rho_\varepsilon$ and $\widetilde{\mu}$
defined by
\begin{equation*}
    \gamma(\mathrm{d}\vx,\mathrm{d}\vz)
    :=
    \rho_\varepsilon(\mathrm{d}\vx)\,\delta_{\Pi(\vx)}(\mathrm{d}\vz),
\end{equation*}
where $\delta_{\Pi(\vx)}$ is the Dirac measure at $\Pi(\vx)$.
By construction, the first marginal of $\gamma$ is $\rho_\varepsilon$
and the second marginal is $\widetilde{\mu}$, so $\gamma\in\Gamma(\rho_\varepsilon,\widetilde{\mu})$.
Then
\begin{align*}
    W_1(\rho_\varepsilon,\widetilde{\mu})
    &\le
    \int \|\vx-\vz\|_2\,\mathrm{d}\gamma(\vx,\vz) \\
    &=
    \int \|\vx-\Pi(\vx)\|_2\,\mathrm{d}\rho_\varepsilon(\vx) \\
    &=
    \mathbb{E}_{\vx\sim\rho_\varepsilon}
    \big[\mathrm{dist}(\vx,\mathcal{M})\big].
\end{align*}

\textbf{Step 2: Calibrating within the manifold.}
We now bound $W_1(\widetilde{\mu},\mu_{\mathrm{tr}})$.
Observe that both $\widetilde{\mu}$ and $\mu_{\mathrm{tr}}$ are supported on $\mathcal{M}$.
For any two probability measures $\nu_1,\nu_2$ supported on a metric space $(\mathcal{M},\|\cdot\|_2)$
with finite diameter $D_{\mathcal{M}}$, the 1-Wasserstein distance satisfies
\begin{equation*}
    W_1(\nu_1,\nu_2)
    \le
    D_{\mathcal{M}}\,\mathrm{TV}(\nu_1,\nu_2),
\end{equation*}
where $\mathrm{TV}(\nu_1,\nu_2)$ is the total variation distance between $\nu_1$ and $\nu_2$.
Indeed, for any coupling $\pi$ of $(X,Y)\sim(\nu_1,\nu_2)$,
\begin{equation*}
    \int \|x-y\|_2\,\mathrm{d}\pi(x,y)
    \le
    D_{\mathcal{M}}\int \mathbf{1}\{x\neq y\}\,\mathrm{d}\pi(x,y)
    =
    D_{\mathcal{M}}\mathbb{P}_\pi[X\neq Y],
\end{equation*}
and taking the infimum over $\pi$ yields
$W_1(\nu_1,\nu_2) \le D_{\mathcal{M}}\inf_\pi \mathbb{P}_\pi[X\neq Y]
= D_{\mathcal{M}}\,\mathrm{TV}(\nu_1,\nu_2)$.
In particular, since $\mathrm{TV}(\nu_1,\nu_2)\le 1$, we have
\begin{equation*}
    W_1(\nu_1,\nu_2) \le D_{\mathcal{M}}.
\end{equation*}

Applying this general bound with $\nu_1=\widetilde{\mu}$ and $\nu_2=\mu_{\mathrm{tr}}$ gives
\begin{equation*}
    W_1(\widetilde{\mu},\mu_{\mathrm{tr}})
    \le D_{\mathcal{M}}.
\end{equation*}
We can therefore define the calibration term
\begin{equation*}
    C_{\mathrm{calib}}
    :=
    D_{\mathcal{M}},
\end{equation*}
which depends only on the internal spread (diameter) of $\mathcal{M}$ under $\mu_{\mathrm{tr}}$.

\textbf{Step 3: Triangle inequality.}
Finally, by the triangle inequality for the Wasserstein distance,
\begin{equation*}
    W_1(\rho_\varepsilon,\mu_{\mathrm{tr}})
    \le
    W_1(\rho_\varepsilon,\widetilde{\mu})
    +
    W_1(\widetilde{\mu},\mu_{\mathrm{tr}}).
\end{equation*}
Combining the bounds from Steps 1 and 2 yields
\begin{equation*}
    W_1(\rho_\varepsilon,\mu_{\mathrm{tr}})
    \le
    \mathbb{E}_{\vx\sim\rho_\varepsilon}
    \big[\mathrm{dist}(\vx,\mathcal{M})\big]
    +
    C_{\mathrm{calib}},
\end{equation*}
which is exactly the claimed inequality.
\end{proof}

This lemma reveals that the primary contributor to the Wasserstein distance is the expected distance of near-optimal designs from the data manifold. The calibration term $C_{\mathrm{calib}}$ is typically small in practice, as long as the training pair sampler sufficiently covers $\mathcal{M}$.
Combining the above lemma with Corollary~\ref{corollary:marginal} immediately yields our next result.

Under the assumptions of Corollary~\ref{corollary:marginal} states that,
with probability at least $1-\delta$, every $\theta\in\Theta$ satisfies
\begin{equation}
    \mathcal{E}^{\mathrm{rank}}_\varepsilon(\theta)
    \;\le\;
    \widehat{\mathcal{E}}^{\mathrm{rank}}_{\mathrm{tr}}(\theta)
    +
    2\,\hat{\mathfrak{R}}_Z(\mathcal{F})
    +
    3\sqrt{\frac{\log(1/\delta)}{2m}}
    +
    L\Big(
        W_1(\rho_\varepsilon,\mu_{\mathrm{tr}})
        +
        W_1(\rho_{>\varepsilon},\nu_{\mathrm{tr}})
    \Big).
    \label{eq:geom-start}
\end{equation}

Next, we invoke Lemma~\ref{lemma:manifold_distance}.  
That lemma gives a geometric upper bound on the Wasserstein distance between the
near-optimal distribution $\rho_\varepsilon$ and the training marginal $\mu_{\mathrm{tr}}$:
\begin{equation}
    W_1(\rho_\varepsilon,\mu_{\mathrm{tr}})
    \;\le\;
    \mathbb{E}_{\vx\sim\rho_\varepsilon}
    \big[\mathrm{dist}(\vx,\mathcal{M})\big]
    +
    C_{\mathrm{calib}},
    \label{eq:geom-manifold}
\end{equation}
where $\mathcal{M}$ is the data manifold (support of $\mu_{\mathrm{tr}}$),
$\mathrm{dist}(\vx,\mathcal{M})$ is the Euclidean distance from $\vx$ to $\mathcal{M}$,
and $C_{\mathrm{calib}}$ is a constant depending only on the internal spread of
$\mathcal{M}$.

Substituting the inequality~\eqref{eq:geom-manifold} into the bound~\eqref{eq:geom-start}
yields
\begin{align*}
    \mathcal{E}^{\mathrm{rank}}_\varepsilon(\theta)
    &\le
    \widehat{\mathcal{E}}^{\mathrm{rank}}_{\mathrm{tr}}(\theta)
    +
    2\,\hat{\mathfrak{R}}_Z(\mathcal{F})
    +
    3\sqrt{\frac{\log(1/\delta)}{2m}}
    +
    L\Big(
        \mathbb{E}_{\vx\sim\rho_\varepsilon}
        \big[\mathrm{dist}(\vx,\mathcal{M})\big]
        +
        C_{\mathrm{calib}}
        +
        W_1(\rho_{>\varepsilon},\nu_{\mathrm{tr}})
    \Big).
\end{align*}
This is exactly the inequality stated in Corollary~\ref{corollary:geometry}.
Therefore, the claimed geometric OOD-aware generalization bound holds.
\end{proof}

\section{Additional Experimental Results}
\label{appendix:additional-exp}

In this section, we present additional experimental results to analyze the sensitivity of the proposed Distribution-Aware Ranking (DAR) method with respect to its key hyperparameters, including the quantile level $\varepsilon$, the intra-region sampling ratio $\lambda$, and the ranking margin $\beta$. Across all Design-Bench tasks, we observe that DAR exhibits stable performance over a reasonably wide range of hyperparameter choices, indicating that its effectiveness does not rely on delicate tuning. In particular, varying $\varepsilon$ primarily trades off between coverage of near-optimal regions and robustness to noise: moderate values (e.g., $\varepsilon \in [0.1, 0.3]$) consistently achieve strong performance, while overly small quantiles lead to insufficient training signal and overly large quantiles dilute the focus on high-quality designs. The intra-region ratio $\lambda$ plays a secondary but regularizing role: small positive values improve stability by encouraging local consistency within the near-optimal region, whereas setting $\lambda=0$ slightly degrades performance due to over-emphasis on cross-region comparisons. Finally, the margin parameter $\beta$ mainly affects optimization dynamics through gradient scale; we find that intermediate margins yield the best balance between convergence speed and ranking robustness. Overall, these results corroborate the theoretical analysis in Section~3, showing that performance gains stem primarily from reducing distributional mismatch rather than precise hyperparameter calibration, and further demonstrate the practical robustness of DAR for offline MBO.

% ============== 表格1: Interpolation 消融 (固定 margin=0.4, pos=0.2) ==============
\begin{table}[ht]
\centering
\caption{Ablation study on interpolation coefficient $\lambda$.}
\label{tab:ablation_int}
\resizebox{\textwidth}{!}{%
\begin{tabular}{l|ccccc|c}
\toprule
$\lambda$ & Ant & D'Kitty & Superconductor & TF-Bind-8 & TF-Bind-10 & Avg. Rank \\
\midrule
0.0 & 0.900 $\pm$ 0.032 & \textbf{\textcolor{violet}{0.950 $\pm$ 0.012}} & 0.496 $\pm$ 0.009 & \textbf{\textcolor{blue}{0.989 $\pm$ 0.003}} & 0.670 $\pm$ 0.069 & \textbf{\textcolor{violet}{3.4}} \\
0.1 & \textbf{\textcolor{blue}{0.941 $\pm$ 0.013}} & \textbf{\textcolor{blue}{0.959 $\pm$ 0.006}} & \textbf{\textcolor{violet}{0.506 $\pm$ 0.008}} & \textbf{\textcolor{violet}{0.988 $\pm$ 0.008}} & \textbf{\textcolor{violet}{0.686 $\pm$ 0.056}} & \textbf{\textcolor{blue}{1.6}} \\
0.2 & 0.914 $\pm$ 0.032 & 0.949 $\pm$ 0.007 & \textbf{\textcolor{blue}{0.507 $\pm$ 0.009}} & 0.979 $\pm$ 0.019 & 0.674 $\pm$ 0.058 & 3.4 \\
0.3 & 0.923 $\pm$ 0.024 & 0.946 $\pm$ 0.006 & 0.496 $\pm$ 0.029 & 0.984 $\pm$ 0.008 & 0.647 $\pm$ 0.036 & 3.6 \\
0.4 & 0.919 $\pm$ 0.013 & 0.943 $\pm$ 0.005 & 0.476 $\pm$ 0.047 & 0.978 $\pm$ 0.014 & 0.642 $\pm$ 0.035 & 5.4 \\
0.5 & \textbf{\textcolor{violet}{0.937 $\pm$ 0.018}} & 0.943 $\pm$ 0.005 & 0.473 $\pm$ 0.050 & 0.981 $\pm$ 0.008 & \textbf{\textcolor{blue}{0.699 $\pm$ 0.096}} & 3.6 \\
\bottomrule
\end{tabular}%
}
\end{table}

% ============== 表格2: Margin 消融 (固定 int=0.1, pos=0.2) ==============
\begin{table}[ht]
\centering
\caption{Ablation study on margin $\beta$.}
\label{tab:ablation_margin}
\resizebox{\textwidth}{!}{%
\begin{tabular}{l|ccccc|c}
\toprule
$\beta$ & Ant & D'Kitty & Superconductor & TF-Bind-8 & TF-Bind-10 & Avg. Rank \\
\midrule
0.0 & \textbf{\textcolor{violet}{0.936 $\pm$ 0.025}} & \textbf{\textcolor{blue}{0.960 $\pm$ 0.011}} & \textbf{\textcolor{violet}{0.505 $\pm$ 0.008}} & 0.982 $\pm$ 0.015 & \textbf{\textcolor{violet}{0.672 $\pm$ 0.060}} & \textbf{\textcolor{violet}{2.4}} \\
0.1 & 0.919 $\pm$ 0.036 & 0.951 $\pm$ 0.008 & 0.498 $\pm$ 0.011 & 0.975 $\pm$ 0.017 & 0.640 $\pm$ 0.045 & 5.6 \\
0.2 & 0.920 $\pm$ 0.022 & 0.953 $\pm$ 0.010 & 0.504 $\pm$ 0.005 & 0.983 $\pm$ 0.010 & 0.641 $\pm$ 0.029 & 3.8 \\
0.3 & 0.927 $\pm$ 0.022 & 0.952 $\pm$ 0.008 & 0.503 $\pm$ 0.007 & \textbf{\textcolor{violet}{0.987 $\pm$ 0.004}} & 0.659 $\pm$ 0.025 & 3.6 \\
0.4 & \textbf{\textcolor{blue}{0.941 $\pm$ 0.013}} & \textbf{\textcolor{violet}{0.959 $\pm$ 0.006}} & \textbf{\textcolor{blue}{0.506 $\pm$ 0.008}} & \textbf{\textcolor{blue}{0.988 $\pm$ 0.008}} & \textbf{\textcolor{blue}{0.686 $\pm$ 0.056}} & \textbf{\textcolor{blue}{1.2}} \\
0.5 & 0.931 $\pm$ 0.012 & 0.948 $\pm$ 0.005 & 0.503 $\pm$ 0.006 & 0.983 $\pm$ 0.016 & 0.634 $\pm$ 0.028 & 4.4 \\
\bottomrule
\end{tabular}%
}
\end{table}

% ============== 表格3: Positive Ratio 消融 (固定 int=0.1, margin=0.4) ==============
\begin{table}[ht]
\centering
\caption{Ablation study on positive ratio $\varepsilon$.}
\label{tab:ablation_pos}
\resizebox{\textwidth}{!}{%
\begin{tabular}{l|ccccc|c}
\toprule
$\varepsilon$ & Ant & D'Kitty & Superconductor & TF-Bind-8 & TF-Bind-10 & Avg. Rank \\
\midrule
0.05 & 0.924 $\pm$ 0.015 & 0.944 $\pm$ 0.006 & 0.405 $\pm$ 0.037 & 0.941 $\pm$ 0.018 & 0.654 $\pm$ 0.029 & 4.6 \\
0.1 & 0.930 $\pm$ 0.012 & 0.942 $\pm$ 0.008 & 0.498 $\pm$ 0.012 & 0.961 $\pm$ 0.033 & 0.674 $\pm$ 0.054 & 4.0 \\
0.15 & \textbf{\textcolor{violet}{0.931 $\pm$ 0.021}} & 0.946 $\pm$ 0.007 & \textbf{\textcolor{violet}{0.505 $\pm$ 0.011}} & 0.973 $\pm$ 0.021 & 0.674 $\pm$ 0.055 & 2.6 \\
0.2 & \textbf{\textcolor{blue}{0.941 $\pm$ 0.013}} & \textbf{\textcolor{violet}{0.959 $\pm$ 0.006}} & \textbf{\textcolor{blue}{0.506 $\pm$ 0.008}} & \textbf{\textcolor{blue}{0.988 $\pm$ 0.008}} & \textbf{\textcolor{violet}{0.686 $\pm$ 0.056}} & \textbf{\textcolor{blue}{1.4}} \\
0.25 & 0.912 $\pm$ 0.018 & \textbf{\textcolor{blue}{0.959 $\pm$ 0.008}} & 0.503 $\pm$ 0.009 & \textbf{\textcolor{violet}{0.978 $\pm$ 0.013}} & \textbf{\textcolor{blue}{0.689 $\pm$ 0.155}} & \textbf{\textcolor{violet}{2.4}} \\
\bottomrule
\end{tabular}%
}
\end{table}

\end{document}